\documentclass[aos]{imsart}

\RequirePackage{amsthm,amsmath,amsfonts,amssymb}
\RequirePackage[numbers]{natbib}
\RequirePackage[colorlinks,citecolor=blue,urlcolor=blue]{hyperref}
\RequirePackage{graphicx}

\startlocaldefs
\theoremstyle{plain}

\newtheorem{theorem}{Theorem}[section]
\newtheorem{lemma}[theorem]{Lemma}
\theoremstyle{remark}
\newtheorem{definition}[theorem]{Definition}
\newtheorem*{example}{Example}

\usepackage{upgreek} %
\usepackage{xspace} %
\usepackage{enumitem} %
\usepackage{etoc} %
\usepackage{thmtools,thm-restate}
\usepackage[capitalize]{cleveref}  

\usepackage{autonum}

\newtheorem{remark}{Remark}

\newtheorem{assumption}{Assumption}
\newtheorem{corollary}{Corollary}
\newtheorem{proposition}{Proposition}

\newtheorem{myexample}{Example}

\newtheorem{myproposition}{Proposition}

\newtheorem{mycorollary}{Corollary}

\newtheorem{mylemma}{Lemma}

\newtheorem{mydefinition}{Definition}

\newtheorem{myremark}{Remark}

\newtheorem{myassumption}{Assumption}

\crefname{appendix}{App.}{App.}
\crefname{subsubsubappendix}{App.}{App.}
\crefname{equation}{}{}
\crefname{lemma}{Lem.}{Lem.}
\crefname{theorem}{Thm.}{Thm.}
\crefname{Corollary}{Cor.}{Cors.}
\crefname{algorithm}{Alg.}{Algs.}

\crefname{section}{Sec.}{Sec.}
\crefname{table}{Tab.}{Tab.}
\crefname{remark}{Rem.}{Rem.}
\crefname{definition}{Def.}{Def.}
\crefname{Proposition}{Prop.}{Prop.}
\crefname{myremark}{Rem.}{Rem.}
\crefname{mylemma}{Lem.}{Lem.}
\crefname{mydefinition}{Def.}{Defs.}
\crefname{myproposition}{Prop.}{Prop.}
\crefname{mycorollary}{Cor.}{Cors.}
\crefname{myassumption}{Assum.}{Assum.}
\crefname{figure}{Fig.}{Fig.}
\crefname{myexample}{Ex.}{Ex.}
\crefname{enumi}{}{}
\crefname{name}{}{} %
\DeclareMathAlphabet{\mathbsf}{OT1}{cmss}{bx}{n}%
\DeclareMathAlphabet{\mathssf}{OT1}{cmss}{m}{sl}%

\newcommand{\drname}{DR-NN\xspace}

 \newcommand{\mprod}[1][\n, j, \t, \t']{M_{#1}}
 \newcommand{\noisehat}{\what{\vareps}}
 \newcommand{\noisetil}{\wtil{\vareps}}
 \newcommand{\rhounitstar}[1][j]{\nbrdist^{\unittag}_{\n, \t}(#1)}
 \newcommand{\rhotimestar}[1][\t']{\nbrdist^{\timetag}_{\n, \t}(#1)}

\newcommand{\rhounit}[1][\n, \t]{ \what{\nbrdist}^{\unittag}_{#1}(j)}
\newcommand{\rhotime}[1][\n, \t]{ \what{\nbrdist}^{\timetag}_{#1}(\t')}
\newcommand{\nmiss}[1][\n, \t]{\miss[#1]}

\newcommand{\uhat}[1][\n]{\what{\lunit[]}_{#1}}
\newcommand{\vhat}[1][\t]{\what{\ltime[]}_{#1}}

\newcommand{\M}{M}
\newcommand{\munit}{\M_{\unittag}}
\newcommand{\mtime}{M_{\timetag}}

\newcommand{\N}{N}

\newcommand{\inprob}{\quad\stackrel{P}{\longrightarrow}\quad}
\newcommand{\sinprob}{\stackrel{P}{\longrightarrow}}

\newcommand{\errterm}[1][\delta]{e_{#1}}
\newcommand{\errtermf}[1][\delta]{e_{#1}}

\newcommand{\errtwo}{\errterm[p, \delta, \N]}

\newcommand{\Sigu}{\Sigma_{\unittag}}
\newcommand{\Sigv}{\Sigma_{\timetag}}

\newcommand{\vconst}{c_{\ltime[]}}
\newcommand{\uconst}{c_{\lunit[]}}
\newcommand{\nconst}{c_{\noise}}

\newcommand{\obsvar}{Y}
\newcommand{\trueobsvar}{\theta}

\newcommand{\miss}[1][\n, \t]{A_{#1}}

\newcommand{\obs}[1][\n, \t]{\obsvar_{#1}}
\newcommand{\trueobs}[1][\n,\t]{\trueobsvar_{#1}}
\newcommand{\noiseobs}[1][\n, \t]{\noise_{#1}}
\newcommand{\estobs}[1][\n,\t, \threshold]{\what{\trueobsvar}_{#1}}
\newcommand{\estobsunit}[1][\n,\t, \threshold_1]{\what{\trueobsvar}_{#1}^{ \unittag}}
\newcommand{\estobstime}[1][\n,\t, \threshold_2]{\what{\trueobsvar}_{#1}^{ \timetag}}
\newcommand{\estobsdr}[1][\n,\t, \mbi{\threshold}]{\what{\trueobsvar}_{#1}^{\drtag}}
\newcommand{\estobsdrs}[1][\n,\t, \mbi{\threshold}]{\what{\trueobsvar}_{#1}^{\drstag}}

\newcommand{\lfun}{f}
\newcommand{\lunit}[1][\n]{u_{#1}}
\newcommand{\ltime}[1][\t]{v_{#1}}

\newcommand{\ybnd}{D}

\newcommand{\unittag}{\trm{unit}}
\newcommand{\timetag}{\trm{time}}
\newcommand{\drtag}{\trm{DR}}
\newcommand{\drstag}{\trm{DR}}

\newcommand{\estnbrnotag}{\mbf{S}}
\newcommand{\nestnbrnotag}{S}
\newcommand{\estnbr}[1][\n, \t, \threshold_1]{\estnbrnotag^{\unittag}_{#1}}
\newcommand{\nestnbr}[1][\n, \t, \threshold_1]{\nestnbrnotag^{\unittag}_{#1}}
\newcommand{\testnbr}[1][\n, \t, \threshold_2]{\estnbrnotag^{\timetag}_{#1}}
\newcommand{\ntestnbr}[1][\n, \t, \threshold_2]{\nestnbrnotag^{\timetag}_{#1}}

\newcommand{\nbrdist}{\rho}
\newcommand{\threshold}{\eta}
\newcommand{\bthreshold}{\mbi{\threshold}}

\newcommand{\suchthat}{\ \big \vert \ }

\newcommand{\nidx}[1][{i}]{#1}
\newcommand{\tidx}[1][{t}]{#1}

\newcommand{\actionset}{\mc A}

\newcommand{\ntimes}[1][T]{#1}

\newcommand{\noise}{\vareps}

\newcommand{\T}{\ntimes}

\newcommand{\n}{\nidx}
\renewcommand{\t}{\tidx}

\newcommand{\order}{\mc{O}}

\newcommand{\mfk}{\mathfrak}

\newcommand{\indicator}{\mbf 1}

\newcommand{\snorm}[1]{\Vert #1 \Vert}
\newcommand{\sinfnorm}[1]{\snorm{#1}_\infty}

\newcommand{\m}{m}

\newcommand{\sless}[1]{\stackrel{#1}{\leq}}

\newcommand{\seq}[1]{\stackrel{#1}{=}}

\newcommand{\x}{x}

\newcommand{\axi}[1][i]{\x_{#1}}

\renewcommand{\l}{\ell}

\newcommand{\eps}{\epsilon}

\newcommand{\vareps}{\varepsilon}

\newcommand{\ltwonorm}[1]{\norm{#1}_{2}}
\newcommand{\stwonorm}[1]{\Vert{#1}\Vert_{2}}

\newcommand{\eventnotag}{\mc{E}}
\newcommand{\event}[1][]{\eventnotag_{#1}}

\newcommand{\wtil}[1]{\widetilde{#1}}

\newcommand{\pseqxn}[1][n]{(\axi[i])_{i\geq 1}} %
\newcommand{\pseqxnn}[1][n]{(\axi[i])_{i=1}^n} %

\newcommand{\brackets}[1]{\left[ #1 \right]}

\newcommand{\bigbrackets}[1]{\big[ #1 \big]}

\newcommand{\parenth}[1]{\left( #1 \right)}
\newcommand{\sparenth}[1]{( #1 )}
\newcommand{\bigparenth}[1]{\big( #1 \big)}
\newcommand{\biggparenth}[1]{\bigg( #1 \bigg)}
\newcommand{\sbraces}[1]{\{ #1  \}}
\newcommand{\braces}[1]{\left\{ #1 \right \}}

\newcommand{\abss}[1]{\left| #1 \right |}
\newcommand{\sabss}[1]{| #1  |}
\newcommand{\angles}[1]{\left\langle #1 \right \rangle}
\newcommand{\sangles}[1]{\langle #1 \rangle}

\newcommand{\tp}{^\top}
\newcommand{\inv}{^{-1}}
\newcommand{\real}{\ensuremath{\mathbb{R}}}

\def\balign#1\ealign{\begin{align}#1\end{align}}
\def\baligns#1\ealigns{\begin{align*}#1\end{align*}}
\def\balignat#1\ealign{\begin{alignat}#1\end{alignat}}
\def\balignats#1\ealigns{\begin{alignat*}#1\end{alignat*}}
\def\bitemize#1\eitemize{\begin{itemize}#1\end{itemize}}
\def\benumerate#1\eenumerate{\begin{enumerate}#1\end{enumerate}}

\newenvironment{talign*}
 {\let\displaystyle\textstyle\csname align*\endcsname}
 {\endalign}
\newenvironment{talign}
 {\let\displaystyle\textstyle\csname align\endcsname}
 {\endalign}

\def\balignst#1\ealignst{\begin{talign*}#1\end{talign*}}
\def\balignt#1\ealignt{\begin{talign}#1\end{talign}}
\newcommand{\rd}[1]{}

\newcommand{\qtext}[1]{\quad\text{#1}\quad} 
\newcommand{\stext}[1]{\ \text{#1}\ } 

\let\originalleft\left
\let\originalright\right
\renewcommand{\left}{\mathopen{}\mathclose\bgroup\originalleft}
\renewcommand{\right}{\aftergroup\egroup\originalright}

\def\tinycitep*#1{{\tiny\citep*{#1}}}
\def\tinycitealt*#1{{\tiny\citealt*{#1}}}
\def\tinycite*#1{{\tiny\cite*{#1}}}
\def\smallcitep*#1{{\scriptsize\citep*{#1}}}
\def\smallcitealt*#1{{\scriptsize\citealt*{#1}}}
\def\smallcite*#1{{\scriptsize\cite*{#1}}}

\def\mbi#1{\boldsymbol{#1}} %
\def\mbf#1{\mathbf{#1}}
\def\mbb#1{\mathbb{#1}}
\def\mc#1{\mathcal{#1}}
\def\mrm#1{\mathrm{#1}}
\def\trm#1{\textrm{#1}}
\def\tbf#1{\textbf{#1}}

\def\N{\mathbb{N}}
\def\<{\left\langle} %
\def\>{\right\rangle}

\def\implies{\quad\Longrightarrow\quad}

\def\defeq{\triangleq} %
\def\half{\frac{1}{2}}
\def\quarter{\frac{1}{4}}
\def\v#1{\mbi{#1}} %
\def\norm#1{\left\|{#1}\right\|} %
\newcommand{\twonorm}[1]{\norm{#1}_2} %
\newcommand{\opnorm}[1]{\norm{#1}_{\mathrm{op}}} %
\def\what#1{\widehat{#1}}

\def\E{\mbb{E}} %

\def\P{\mbb{P}} %

\def\Var{\mrm{Var}} %

\newcommand{\grad}{\nabla} %
\newcommand{\iid}{\textrm{i.i.d}\xspace}

\ifdefined\nonewproofenvironments\else
\ifdefined\ispres\else
\newenvironment{proof-sketch}{\noindent\textbf{Proof Sketch}
  \hspace*{1em}}{\qed\bigskip\\}
\newenvironment{proof-idea}{\noindent\textbf{Proof Idea}
  \hspace*{1em}}{\qed\bigskip\\}
\newenvironment{proof-of-lemma}[1][{}]{\noindent\textbf{Proof of Lemma {#1}}
  \hspace*{1em}}{\qed\\}
\newenvironment{proof-of-theorem}[1][{}]{\noindent\textbf{Proof of Theorem {#1}}
  \hspace*{1em}}{\qed\\}
\newenvironment{proof-attempt}{\noindent\textbf{Proof Attempt}
  \hspace*{1em}}{\qed\bigskip\\}

\graphicspath{{../code/figs/},{../figs/}}

\renewcommand{\N}{N}
\endlocaldefs

\begin{document}

\begin{frontmatter}
\title{Doubly robust nearest neighbors in factor models}
\runtitle{Doubly robust nearest neighbors}
\begin{center}
\textsc{By Raaz Dwivedi$^{1,2}$, Caleb Chin$^{2}$, Sabina Tomkins$^{3}$ \\ 
Predrag Klasnja$^{2}$, Susan Murphy$^{4}$ and Devavrat Shah$^{5}$ } \\
{\small{\emph{Cornell Tech$^{1}$, Cornell University$^{2}$,
University of Michigan$^{3}$,\\
Harvard University$^{4}$,  Massachusetts Institute of Technology$^{5}$}}}
\end{center}

\begin{abstract}
We introduce and analyze an improved variant of nearest neighbors (NN) for estimation with missing data in latent factor models. We consider a matrix completion problem with missing data, where the $(i, t)$-th entry, when observed, is given by its mean $f(u_i, v_t)$ plus mean-zero noise for an unknown function $f$ and latent factors $u_i$ and $v_t$. Prior NN strategies, like unit-unit NN, for estimating the mean $f(u_i, v_t)$ relies on existence of other rows $j$ with $u_j \approx u_i$. Similarly, time-time NN strategy relies on existence of columns $t'$ with $v_{t'} \approx v_t$. These strategies provide poor performance respectively when similar rows or similar columns are not available. Our estimate is doubly robust to this deficit in two ways: (1) As long as there exist either good row or good column neighbors, our estimate provides a consistent estimate. (2) Furthermore, if both good row and good column neighbors exist, it provides a (near-)quadratic improvement in the non-asymptotic error and admits a significantly narrower asymptotic confidence interval when compared to both unit-unit or time-time NN. 
\end{abstract}

\end{frontmatter}
\newcommand{\vlf}{\mc V}
\newcommand{\ulf}{\mc U} 
\section{Introduction}
Latent factor models are useful tools for statistical inference with missing data. They are popular in the literature on recommender systems~\citep{desrosiers2011comprehensive, koren2015advances, goldberg1992using, linden2003amazon}, e-commerce~\citep{avadhanula2021stochastic, sawant2018contextual}, and panel data problems in econometrics~\citep{abadie1, abadie2, athey2021matrix, arkhangelsky2019synthetic, RSC}, and more recently have been utilized for personalized decision-making in health~\citep{nahum2018just, klasnja2015microrandomized, boruvka2018assessing, klasnja2019efficacy}, including sequential decision-making with adaptive algorithms like bandits~\citep{tewari2017ads, russo2018tutorial, lattimore2020bandit}. One of the canonical problem formulations is the matrix completion problem~\citep{CandesTao10, Recht11, KeshavanMontanariOh10a, NegahbanWainwright11, Chatterjee15, ieee_matrix_completion_overview}: Given a matrix of noisy outcomes with missing entries, estimate the denoised outcome for each (missing or not) entry in the matrix.  In particular, our goal is to estimate the entries of a matrix $\Theta \in \real^{(\N+1)\times (\T+1)}$ with its $(i, t)$-th entry denoted by $\trueobs$, where for each $i\in[\N+1]$ and $\t\in[\T+1]$, we observe
\begin{talign}
\label{eq:model_mc}
    \obs = \begin{cases}
    \trueobs + \noiseobs &\qtext{if} \miss = 1 \\ 
    \star &\qtext{if} \miss = 0
    \end{cases},
\end{talign}
where $ \noiseobs$ denotes real-valued noise while $\miss$ is a binary variable denoting which entries of the matrix are observed, and $\star$ denotes that the entry is not observed. For recommender systems~\citep{movielens, funk2006netflix, koren08}, $\trueobs$ might denote a user $i$'s true rating for a product $t$, and a noisy value $\obs$ is available only when the user rates it ($\miss = 1$). In an advertisement application~\citep{schwartz2017customer, avadhanula2021stochastic}, $\obs$ might denote whether user $i$ bought a product $t$ if a particular ad ($\miss=1$) is shown to the user. In digital health~\citep{klasnja2019efficacy, liao2020personalized, klasnja2015microrandomized}, $\trueobs$ might denote a health outcome, e.g., step count, or heart rate, for user $i$ at some decision time $t$, and $\miss=1$ denotes whether a binary intervention, e.g., a reminder for physical activity, was delivered at that time. For clarity in our discussion (and partly motivated by the empirical application considered later), henceforth, we refer to the rows of the matrix as units and columns as time points.

Without any assumptions, estimating the matrix $\Theta$ is an impossible task due to more number of unknown parameters than the number of (noisy) observations. A common modeling approach to make the task feasible is via a  latent factor model. In particular, one assumes that the entries of $\Theta$ admit a (non-linear) factorization in terms of two sets of (low) $d$-dimensional latent factors---$\ulf \defeq \sbraces{\lunit, \n \in[\N+1]}$ and $\vlf \defeq \sbraces{\ltime, \t\in[\T+1]}$ such that
 \begin{align}
 \label{eq:bilinear}
     \trueobs = \lfun(\lunit, \ltime) \qtext{for} \n \in[\N\!+\!1], \t\in[\T\!+\!1].
 \end{align}
Given our naming convention, we call the latent factors $\ulf$ and $\vlf$  to as unit and time latent factors respectively. The factorization~\cref{eq:bilinear} essentially states that the outcome of interest $\trueobs$ depends on two latent/hidden factors the unit $\n$'s factor $\lunit$ and the time $\t$'s factor $\ltime$.
The model~\cref{eq:bilinear} with bilinear $f$, i.e., $f(u, v) = \angles{u, v}$ so that $\Theta = \ulf \vlf\tp$, is ubiquitous in matrix completion literature~\cite{CandesTao10, Recht11, KeshavanMontanariOh10a, NegahbanWainwright11, softimpute, Chatterjee15, cai2016matrix} as well as the works in panel data settings~\cite{dwivedi2022counterfactual,li2019nearest,agarwal2021synthetic,arkhangelsky2019synthetic, athey2021matrix, bai2019matrix}.
 However, when $f$ is non-linear and unknown, relatively fewer strategies are known to work~\citep{WTN, pmf, koren08, udell_low_rank}. 

 A popular approach for inference with non-linear factor models is based on nearest neighbors (NN). Such estimators have shown empirically favorable performance in numerous applications in collaborative filtering, recommender systems~\cite{LeeLiShahSong16, chen2018explaining, desrosiers2011comprehensive} and counterfactual inference in panel data settings~\cite{dwivedi2022counterfactual, abadie2024doubly, choi2024counterfactual}. The estimates often come with entry-wise guarantees instead of averaged error guarantee, i.e., suitable bounds on $\snorm{\what \Theta-\Theta}_{\max}$ are available for NN-based estimates~\citep{LeeLiShahSong16, li2019nearest, dwivedi2022counterfactual}, while typical matrix completion methods provide a bound on $\snorm{\what \Theta-\Theta}_{F}$~\citep{CandesTao10, Recht11, NegahbanWainwright11}. Consequently, the NN approach is especially suitable for tasks requiring fine-grained estimates with statistical guarantees, e.g., for assessing the quality of personalized recommendations in e-commerce~\cite{LeeLiShahSong16} or the quality of interventions in digital health~\cite{dwivedi2022counterfactual, klasnja2019efficacy}. 

For the matrix completion task, there are two natural variants of NN when trying to estimate $\theta_{i, t}$: (1) \emph{unit nearest neighbor} (unit-NN)  and (2) \emph{time nearest neighbor} (time-NN) depending on how the neighbors are defined.\footnote{These two estimators are often called, user-user NN and item-item NN respectively, in the literature focusing on recommender system applications.}  
Let us briefly summarize the unit-NN approach, which involves two key steps. First, one computes the \emph{unit neighbors} 
by comparing the observed outcomes for unit $i$ with that of the other units. Next, one averages the observed outcomes of these unit neighbors at time $t$ to produce the unit-NN estimate. The time-NN is defined analogously with the roles of units and time flipped. Prior works~\cite{li2019nearest, dwivedi2022counterfactual, sadhukhan2025minimax, feitelberg2024distributional} show that the quality of unit-NN (resp. time-NN) estimate depends roughly on the number of unit (resp. time) neighbors and how well these neighbors' factors approximate the factor for unit $i$ (resp. time $t$).
These NN estimates come with strong entry-wise guarantees with non-linear factor model and are interpretable but are known to provide poor performance when the underlying factors are heterogeneous so that the neighbors do not provide a good approximation. The starting point of this work is to address this deficit of unit or time-NN estimates and make them more robust to lack of neighbors or equivalently the heterogeneity across latent factors. 

\paragraph{Our contributions}
We present a new NN estimate that brings together the unit and time-NN estimates to provide a significant improvement for providing entry-wise guarantee for estimating $\Theta$. First, we show that if both unit and time-NN provide a non-trivial non-asymptotic error, the new estimate provides an error that is strictly better than both of them. Moreover, the new estimate provides a non-trivial error decay rate as long as either the unit-NN or time-NN estimate provides a non-trivial error rate. We call this new estimate \emph{doubly robust nearest neighbor estimate} (\drname) as its performance is doubly robust to the performance of the two vanilla NN approaches. More precisely, for bilinear factor model and with suitable generative assumptions on the latent factors, we show that DR-NN error rate is a near-quadratic improvement over the error rates of the vanilla NN estimates. This improvement also translates to a near-quadratic improvement in the width of the asymptotic confidence intervals for DR-NN estimate when compared the vanilla NN estimates. We also show that our construction extends to a ``triply robust'' NN for tensor factor models. The doubly-robust structure of our estimate is related to the line of work on doubly robust estimation in causal inference~\citep{robins2000marginal, cao2009improving, jiang2016doubly} and, in the factor model setting, to the concurrent and follow-up work of Abadie et al.~\cite{abadie2024doubly}, who establish doubly robust inference for causal effects from observational data in a related model. More recent work has extended NN-based matrix completion to distributional estimation~\citep{feitelberg2024distributional}, counterfactual distribution learning~\citep{choi2024counterfactual}, and achieves minimax-optimal rates for matrix completion~\citep{sadhukhan2025minimax, sadhukhan2025adaptive}.

\paragraph{Organization} We introduce the problem set-up and algorithm in \cref{sec:problem_algorithm}, followed by our main results in \cref{sec:main_results}. We then discuss our contributions in the context of related work in \cref{sec:possible_extensions}, 
and conclude with a discussion in \cref{sec:conclusion}. We provide an intuitive construction that also illustrates justifying the improved performance of DR-NN estimate in \cref{sec:algorithm} and defer all the proofs to appendix.
\newcommand{\bilin}{Latent factorization of mean parameters}
\newcommand{\zeromean}{\iid zero mean noise}
\newcommand{\lamunit}{\lambda_{\unittag}}
\newcommand{\lamtime}{\lambda_{\timetag}}

\newcommand{\p}{p}
\newcommand{\thresop}[1][1]{\threshold_{#1}'}
\newcommand{\threstp}[1][2]{\thresop[#1]}

\newcommand{\unitnbr}[1][\thresop]{\mbf{S}_{\n, #1} }
\newcommand{\nunitnbr}[1][\thresop]{S_{\n, #1}}
\newcommand{\timenbr}[1][\threstp]{\mbf{S}_{\t, #1} }
\newcommand{\ntimenbr}[1][\threstp]{S_{\t, #1} }
\newcommand{\dhalf}[1][2]{\frac{d}{#1}}
\newcommand{\effnn}{J_{\n,\t, \mbi{\threshold}}}
\section{Problem set-up and algorithm}
\label{sec:problem_algorithm}

In \cref{sec:data_generation}, we formally describe the data generating mechanism, and the structural assumptions on the counterfactual means, followed by the description of the nearest neighbors algorithm proposed for estimating these means (along with confidence intervals) in \cref{sec:algorithm}. 
\subsection{Data generating mechanism}
\label{sec:data_generation}

Our observation model is given by \cref{eq:model_mc} and our goal is to design a robust NN algorithm for estimating $\trueobs$. We now state our main assumptions on the mean parameters, noise variables and the missingness pattern. Notably, our general guarantee is stated conditional on $\ulf, \vlf$ and does not make any distributional assumptions on these latent factors.

\begin{assumption}[\bilin]
\label{assum:nonlinear}
    Conditioned on the latent factors $\ulf$ and $\vlf$, the mean parameters satisfy $\trueobs =  \lfun(\lunit, \ltime)$.
\end{assumption}
\begin{assumption}[Missingness pattern]
\label{assum:mcar}
    Conditioned on the latent factors $\ulf$ and $\vlf$, the random variables $\sbraces{\miss}$ are drawn \iid from Bernoulli$(p)$ distribution, and are independent of all other randomness.
\end{assumption}
\begin{assumption}[Bounded noise]
    \label{assum:noise}
    Conditioned on $\mc U$ and $\mc V$, the noise variables $\sbraces{\noiseobs}$ are drawn \iid and independent of $\sbraces{\miss}$ and satisfy $\E[\noiseobs \vert \mc U, \mc V]=0$, $\Var(\noiseobs)=\sigma^2$ and $\noiseobs \leq \nconst$ almost surely.
\end{assumption}

All the assumptions are standard in the literature. \cref{assum:mcar} is often known as the MCAR (missing completely at random) assumption~\citep{rubin1976, seaman2013meant, little2019statistical} and is standard in the literature for nearest neighbors for matrix completion~\citep{li2019nearest, dwivedi2022counterfactual, sadhukhan2025minimax}. The bounded noise assumption is made to simplify the analysis and can be easily relaxed to sub-Gaussian noise for high probability non-asymptotic analysis (\cref{thm:anytime_bound}) or noise with bounded high-order moments for asymptotic analysis (\cref{thm:asymp}).

\subsection{Algorithm}
\label{sec:algorithm}

    We first sketch an intuitive construction of DR-NN in \cref{sub:intuition}, followed by a formal description in \cref{sub:formal_algo}.

    \subsubsection{An intuitive construction of DR-NN}
    \label{sub:intuition}
    The notation, and definitions in this section should be treated as informal versions of the estimates defined formally in \cref{sub:formal_algo}. We abuse notation to provide a parallel description behind a constructive intuition in building the improved nearest neighbor estimate.

    Here we consider an idealized setting with no noise and a bilinear factor models so that $\obs=\trueobs = \angles{\lunit, \ltime}$ for all $i, t$. We assume that we observe all but one outcome $\obs[1, 1]$ and that our goal is to simply estimate the unobserved $\trueobs[1, 1]$. The two vanilla NN estimates are then defined as follows.
    
    \paragraph{Unit-NN}
    First one compares the counterfactuals of unit $1$ with all other units across all time points to identify the set $\estnbr[1] \subset[\N+1]$ such that $\lunit[j] \approx \lunit[1]$ for each $j \in \estnbr[1]$. The unit-NN estimate is then given by averaging the unit neighbors' outcomes at time $\t$:
    \begin{align}
        \estobsunit[1, 1] = \frac{\sum_{j\in\estnbr[1]} \obs[j,1]}{|\estnbr[1]|} = \angles{\frac{\sum_{j\in\estnbr[1]}\lunit[j]}{|\estnbr[1]|}, \ltime[1]} \defeq \angles{\uhat[1], \ltime[1]}.
        \label{eq:uhat_def}
    \end{align}
    The quality of this estimate is then given by
    \begin{align}
    \label{eq:unit}
        \trueobs[1, 1]- \estobsunit[1, 1] = \angles{\lunit[1]-\uhat[1], \ltime[1]} \implies \sabss{\trueobs[1, 1]- \estobsunit[1, 1]} \leq  \stwonorm{\lunit[1]- \uhat[1]} \stwonorm{\ltime[1]}.
    \end{align}

    \paragraph{Time-NN}
    Flipping the roles of time and unit in the unit-NN approach, first we identify the set $\testnbr[1] \subset[\T+1]$ such that $\ltime[\t'] \approx \ltime[1]$ for each $\t' \in \testnbr[1]$. The time-NN estimate is then obtained by averaging the outcomes of unit $1$ across these time neighbors:
    \begin{align}
        \estobstime[1, 1] = \frac{\sum_{\t'\in\testnbr[1]} \obs[1, \t']}{|\testnbr[1]|} = \angles{\lunit[1], \frac{\sum_{\t'\in\testnbr[1]}\ltime[\t']}{|\testnbr[1]|}} \defeq \angles{\lunit[1], \vhat[1]},
        \label{eq:vhat_def}
    \end{align}
    for which the error is given by
    \begin{align}
    \label{eq:time}
        \trueobs[1, 1]- \estobstime[1, 1] = \angles{\lunit[1], \ltime[1]-\vhat[1]}
        \implies \sabss{\trueobs[1, 1]- \estobstime[1, 1]} \leq  \stwonorm{\lunit[1]} \stwonorm{\ltime[1]-\vhat[1]}.
    \end{align}

    \paragraph{Steps towards an improved nearest neighbors estimate}
    In this work, we aim to seek a suitable variant of nearest neighbors estimate, denoted by $\estobsdr[1, 1]$, such that
    \begin{align}
    \label{eq:dr}
        \trueobs[1, 1]-\estobsdr[1, 1] = \angles{\lunit[1]-\uhat[1], \ltime[1]-\vhat[1]}
        \qtext{and}
        \sabss{\trueobs[1, 1]- \estobsdr[1, 1]} \leq  \stwonorm{\lunit[1]-\uhat[1]} \stwonorm{\ltime[1]-\vhat[1]},
    \end{align}
    where $\uhat[1]$ and $\vhat[1]$ are as defined in \cref{eq:uhat_def,eq:vhat_def}.
    Intuitively speaking, such an estimate, owing to its error bounded by the product of the error in estimating $\lunit[1]$, and $\ltime[1]$, would exhibit a better error than the previous $\estobsunit[1, 1]$ and $\estobstime[1, 1]$, and provide a vanishing error as long as either unit or time NN provide a vanishing error. 

    \paragraph{``Incorrect'' guesses}
    One might make two ``natural'' guesses for combining the unit- or time-NN estimates in order to find a better estimator:
    \begin{align}
        \what{\theta}_{1, 1}^{(1)} 
        &\defeq \frac{\estobsunit[1, 1] + \estobstime[1, 1]}{2}
        \seq{(\ref{eq:uhat_def},\ref{eq:vhat_def})} \frac{\angles{\uhat[1],\ltime[1]}+\angles{\lunit[1],\vhat[1]}}{2} \qtext{and} \\ 
        \what{\theta}_{1,1}^{(2)} 
        &\defeq 
        \frac{\sum_{t'\in\testnbr[1]}\sum_{j\in\estnbr[1]} \obs[j,\t']}{|\testnbr[1]|\, |\estnbr[1]|} 
        = \angles{\frac{\sum_{j\in\estnbr[1]}\lunit[j]}{|\estnbr[1]|}, \frac{\sum_{\t'\in\testnbr[1]}\ltime[\t']}{|\testnbr[1]|}}
        \seq{(\ref{eq:uhat_def},\ref{eq:vhat_def})} \angles{\uhat[1],\vhat[1]},
    \end{align}
    neither of which achieve the first equality in \cref{eq:dr}. Moreover, 
    simple algebra reveals that
    \begin{align}
        |\angles{\lunit[1], \ltime[1]} - \what{\theta}_{1, 1}^{(1)} |
        &\leq \frac{ \stwonorm{\lunit[1]} \stwonorm{\ltime[1]-\vhat[1]} +  \stwonorm{\lunit[1]-\uhat[1]} \stwonorm{\ltime[1]}}{2} \stext{and}\\ 
        |\angles{\lunit[1], \ltime[1]} - \what{\theta}_{1, 1}^{(2)} |
         &\leq \stwonorm{\lunit[1]} \stwonorm{\ltime[1]-\vhat[1]} +  \stwonorm{\lunit[1]-\uhat[1]} \stwonorm{\ltime[1]},
    \end{align}
    neither of which are analogous to our desired inequality from display~\cref{eq:dr}.

    \paragraph{The ``correct'' guess} Next, we make a constructive guess for $\estobsdr[1, 1]$. Substituting $\trueobs[1, 1]=\angles{\lunit[1],\ltime[1]}$ and some algebra yields that
    \begin{align}
        \estobsdr[1, 1] &= \angles{\lunit[1], \vhat[1]} + \angles{\uhat[1], \ltime[1]} -  \angles{\uhat[1],\vhat[1]} \label{eq:dr_first}\\
        &\seq{(\ref{eq:uhat_def},\ref{eq:vhat_def})}\frac{\sum_{\t'\in\testnbr[1]} \obs[1, \t']}{|\testnbr[1]|}  + \frac{\sum_{j\in\estnbr[1]} \obs[j, 1]}{|\estnbr[1]|} - \frac{\sum_{t'\in\testnbr[1]}\sum_{j\in\estnbr[1]} \obs[j,\t']}{|\testnbr[1]|\, |\estnbr[1]|} \\ 
        &=\frac{\sum_{j\in\estnbr[1],\t'\in\testnbr[1]} (\obs[1, \t'] + \obs[j, 1]- \obs[j, \t'])} {|\estnbr[1]|\,|\testnbr[1]|},
        \label{eq:dr_guess}
    \end{align}
    which can be easily computed from the observed data. This intuitive construction turns out to be the ``right'' one and can be generalized to the case with noisy observations with missing data as well as a class of non-linear factor models. Of course, the construction here was informal and several key steps remain to be formalized, which we now turn our attention to.

    \subsubsection{Vanilla-NN and DR-NN for missing and noisy data}
    \label{sub:formal_algo}
    We begin by formally describing the two vanilla NN variants and conclude with DR-NN.

    \paragraph{Unit-NN}
    \label{sub:unit_est}
    Given a threshold $\threshold_1$, define the unit neighbors $\estnbr$ as
    \begin{align}
    \label{eq:reliable_nbr}
    \estnbr \defeq \!
    \braces{j \neq\n\! \suchthat \!\rhounit \!\leq\! \threshold_1}
    \,\stext{with}\,
    \rhounit \!\defeq\! \displaystyle  \frac{ \sum_{\t'\neq\t} \miss[\n, \t'] \miss[j, \t'] (\obs[\n,\t']\!-\!\obs[j, \t'])^2}{\sum_{\t'\neq\t} \miss[\n, \t'] \miss[j, \t']}
    \stext{for} j\neq \n.
    \end{align}
    In words, $\rhounit$ denotes the mean squared distance of observed outcomes for unit $\n$ and $j$ at time points except $\t$ and the set $\estnbr$ denotes the units that are within $\threshold_1$ distance from unit $\n$. Then the unit-NN estimate is given by
    \begin{align}
    \label{eq:obs_estimate}
        \estobsunit \defeq 
         \displaystyle\frac{\sum_{j \in
         \estnbr} \miss[j, \t]\obs[j, \t]}{\sum_{j \in
         \estnbr} \miss[j, \t]},
    \end{align}

    \paragraph{Time-NN}
    \label{sub:time_est}
    Given the threshold $\threshold_2$, define the time neighbors $\testnbr$ as
     \begin{align}
        \label{eq:reliable_nbr_time}
        \testnbr \defeq \!
    \braces{\t' \neq\t\! \suchthat \!\rhotime \!\leq\! \threshold_2}
    \,\stext{with}\,
    \rhotime \!\defeq\! \displaystyle  \frac{ \sum_{j\neq\n} \miss[j, \t] \miss[j, \t] (\obs[j,\t]\!-\!\obs[j, \t'])^2}{\sum_{j\neq\n} \miss[j, \t] \miss[j, \t']}
    \stext{for} \t'\neq \t.
    \end{align}
    In words, $\rhotime$ denotes the mean squared distance of observed outcomes for time $\t$ and $\t'$ across units except unit $\n$ and the set $\testnbr$ denotes the time points that are within $\threshold_2$ distance from time $\t$. Then the time-NN estimate is given by
    \begin{align}
    \label{eq:obs_estimate_time}
        \estobstime \defeq 
         \displaystyle\frac{\sum_{\t' \in
         \testnbr} \miss[\n, \t']\obs[\n, \t']}{\sum_{\t' \in
         \testnbr} \miss[\n, \t']} .
    \end{align}

    \paragraph{DR-NN}
    \label{sub:dr_est}
   We now formally define the DR-NN studied in this work. With the definitions~\cref{eq:reliable_nbr,eq:reliable_nbr_time} in place, and using the notation $\mbi{\threshold} \defeq (\threshold_1, \threshold_2)$, the doubly robust estimate $\estobsdr$ can be defined as
    \begin{align}
    \label{eq:obs_estimate_dr}
        \estobsdr \defeq 
         \displaystyle\frac{\sum_{j \in \estnbr} \sum_{\t' \in \testnbr} \miss[\n, \t']  \miss[j, \t]  \miss[j, \t'] (\obs[\n, \t']\!+\!\obs[j, \t] \!-\! \obs[j, \t'])}{\sum_{j \in \estnbr} \sum_{\t' \in \testnbr} \miss[\n, \t']  \miss[j, \t]  \miss[j, \t']}.
    \end{align}
    Put simply, the DR estimate leverages both unit and time neighbors as each term in the numerator of definition~\cref{eq:obs_estimate_dr} is made of the sum of (i) the outcome of unit $\n$ at a similar time point $\t'$, and (ii) the outcome of a similar unit $j$ at time $\t$, (iii) with an offset term of the outcome of unit $j$ and time $\t'$.

    \paragraph{Confidence interval}  Define
    \begin{align}
    \label{def:n_it}
        \effnn \defeq \big(\frac{1} {\Sigma_{\t' \in \testnbr}  \miss[\n, \t']} + \frac{1}{\Sigma_{j \in \estnbr}  \miss[j, \t]} + \frac{1}{\Sigma_{\t' \in \testnbr}\Sigma_{j \in \estnbr}  \miss[\n, \t'] \miss[j, \t']   \miss[j, \t]}\big)\inv.
    \end{align}
    Then our $(1\!-\!\alpha)$ confidence interval for $\trueobs$ is given by
    \begin{align}
    \label{eq:unit_ci}
        \biggparenth{\estobsdr \!\!-\!\! \frac{z_{\frac{\alpha}{2}}\,\what{\sigma}}{\sqrt{\effnn}},\,\,
        \estobsdr \!\!+\!\! \frac{z_{\frac{\alpha}{2}}\,\what{\sigma}}{\sqrt{\effnn}}
        },
    \end{align}
    where $z_{\frac{\alpha}{2}}$ denotes the $1\!-\!\frac{\alpha}{2}$ quantile of standard normal random variable, and $\what{\sigma}^2$ is a consistent estimate of the noise-variance (which can be obtained, e.g., as the validation mean square error like in \citep[App.~E.2]{dwivedi2022counterfactual}). 

    \begin{remark}[Estimates when there are no neighbors]
    \normalfont
    The definitions~\cref{eq:obs_estimate,eq:obs_estimate_time,eq:obs_estimate_dr,def:n_it} assume
    that the corresponding denominator is not $0$ and our analysis provides a guarantee also only for that case (see \cref{rem:no_nbr}). When the denominators in the estimates' displays are $0$, we can simply use $\estobstime = \obs$ if $\miss=1$ else replace the corresponding summations over the neighbors to include all units, time points or both for unit, time and DR-NN respectively.
    \end{remark}

    \begin{remark}[Sample split for theoretical analysis]
    \label{rem:data_split}
    \normalfont
        To prove our results, we assume that the doubly robust estimate was constructed using a data-split. In particular, we use do a $2 \times 2$ data split chosen uniformly at random: (i) Two equal splits of time points $\braces{\t' \in [\T]: \t' \neq \t}$ into  $\mbf T_1$ and $\mbf T_2$, and (ii) two equal splits of units $\braces{j \in [\N]: j \neq \n}$ into  $\mbf N_1$ and $\mbf N_2$. We use time points $\mbf T_1$ to compute the unit distances $\rhounit$ for $j \in \mbf N_1$, and identify the unit-neighbors only in the set $\mbf N_1$, i.e., $\estnbr \subset \mbf N_1$. Next, we use units $\mbf N_2$ to compute the time distances $\rhotime$ for $\t' \in \mbf T_2$, and identify the time-neighbors only in the set $\mc T_2$, i.e., $\testnbr \subset \mbf T_2$. 
        On the one hand, this data split inflates the variance of our estimator. 
        On the other hand, it removes the (hard to analyze) bias that might arise due to data-reuse in estimating the two neighbors sets $\estnbr$ and $\testnbr$ without any data split. 
    \end{remark}

\newcommand{\examplediscrete}{Discrete unit factors\xspace}
\newcommand{\examplecontinuous}{Continuous unit factors\xspace}
\newcommand{\sunit}{\nunitnbr}
\newcommand{\stime}{\ntimenbr}
\newcommand{\bias}{\mbb B}
\newcommand{\variance}{\mbb V}
\newcommand{\biasunit}[1][\threshold_1]{\bias^{\unittag}_{#1}}
\newcommand{\varunit}[1][\threshold_1]{\variance^{\unittag}_{#1}}
\newcommand{\biastime}[1][\threshold_2]{\bias^{\timetag}_{#1}}
\newcommand{\vartime}[1][\threshold_2]{\variance^{\timetag}_{#1}}
\newcommand{\biasdrs}[1][\mbi{\threshold}]{\bias^{\drstag}_{#1}}
\newcommand{\biasdrsasymp}[1][\mbi{\threshold},\trm{asymp}]{\bias^{\drstag}_{#1}}
\newcommand{\vardrs}[1][\mbi{\threshold}]{\variance^{\drstag}_{#1}}
\newcommand{\anytimeboundname}{Non-asymptotic guarantee for $\estobsdrs$ with non-linear factor model}
\newcommand{\bilinearresultname}{Non-asymptotic guarantee $\estobsdrs$ with bilinear factor model}
\newcommand{\otil}{\wtil{\order}}
\newcommand{\ERR}{\mrm{Err}}
\newcommand{\errunit}[1][\threshold_1]{\ensuremath{\ERR_{#1}^{\unittag}}}
\newcommand{\errtime}[1][\threshold_2]{\ensuremath{\ERR_{#1}^{\timetag}}}
\newcommand{\errdrs}[1][\threshold_{\drstag}]{\ensuremath{\ERR_{#1}^{\drstag}}}
\newcommand{\errmin}[1][(\threshold_{\unittag},\threshold_{\timetag})]{\mrm{Min}\ERR^{\unittag\!\!-\!\!\timetag}_{#1}}
\newcommand{\psuccess}{p_{\trm{nn}}}

\section{Main results}
\label{sec:main_results}

    In this section, we present our main result, a unit$\times$time wise non-asymptotic guarantee for the doubly-robust nearest neighbor estimate and then discuss its consequences.

    \subsection{Set-up for analysis}
    Our analysis makes use of the following modeling assumption on the entries in $\Theta$.
    \begin{assumption}[Non-linear factor model]
        \label{assum:non_linear_f}
        The parameters $\Theta$ satisfy
        \begin{align}
            \trueobs\!= f(\lunit, \ltime)
            \qtext{for all} i \in [N+1], t\in[T+1],
        \end{align}
        for a twice continuously-differentiable, Lipschitz function $f$ 
         that satisfies
          $\sup_{i, t}|f(u_i, v_t)|\leq \ybnd$ and $\sup_{u, v}\opnorm{\nabla_{u,v}^2 f(u, v)} \leq L_H$ for some constants $\ybnd$ and $L_H$. Furthermore, there exist universal positive constants $c_{f,1}, c_{f, 2}$, $\alpha_1$, and $\alpha_2$
         such that
        \begin{align}
        \begin{split}
        \label{eq:strong_convexity}
            \rho^{\trm{unit}}(\lunit[], \lunit[]') \defeq \frac1{\T+1}\sum_{t=1}^{\T+1} (f(u, \ltime)-f(u',\ltime))^2 &\geq c_{f,1}\stwonorm{u-u'}^{\alpha_1} 
            \qtext{for all} \lunit[], \lunit[]', \qtext{and} \\ 
            \rho^{\trm{time}}(\ltime[], \ltime[]') \defeq 
            \frac1{\N+1}\sum_{i=1}^{\N+1} (f(\lunit, v)-f(\lunit,v'))^2 &\geq c_{f,2}\stwonorm{v-v'}^{\alpha_2} 
            \qtext{for all} \ltime[], \ltime[]'.
        \end{split} 
        \end{align}
    \end{assumption}
    We note that the condition \cref{eq:strong_convexity} is effectively a convexity condition for the two distances $\rho^{\trm{unit}}$ and $\rho^{\trm{time}}$.
    In \cref{sec:lip_assum_3}, we show that \cref{eq:strong_convexity} holds whenever $\lfun$ is bilinear or satisfy a co-ercivity condition in both arguments.s

    \paragraph{Defining pre-existing set of neighbors} Our general results to follow are stated conditioned on the latent factor $\mc U$ and $\mc V$. However, these guarantees depend on certain quantities (like the number of neighbors apriori available for a given unit and time). We now exemplify some data generating processes on $\mc U$ and $\mc V$, for which it is easy to characterize these quantities. 

    To simplify notation, we drop the superscript $\unittag$ and $\timetag$, and we understand the indices $(\n, j)$ are reserved for units, and $(\t, \t')$ are reserved for time points:
    \begin{align}
    \label{eq:nbr_unit}
       \unitnbr[r] \defeq  \braces{j \neq \n \suchthat \rho^{\trm{unit}}(\lunit, \lunit[j]) \leq r}
        \qtext{and}
       \timenbr[r] \defeq  \braces{ \t' \neq t \suchthat \rho^{\trm{time}}(\ltime, \ltime[\t']) \leq r},
    \end{align}
    and let $\nunitnbr[r] \defeq |\unitnbr[r]|$ and $\ntimenbr[r] \defeq |\timenbr[r]|$. We discuss two distinct examples for probabilistic generating mechanisms on the latent factors such that one can characterize the size of the sets defined in \cref{eq:nbr_unit}. 
    \begin{example}[\examplediscrete]
        \label{example:finite}
        \normalfont
        In this setting, suppose $\ulf$ and $\vlf$ are generated with \iid draws from uniform distributions over $\munit$ and $\mtime$ number of distinct $d$-dimensional vectors (say a subset of $\sbraces{-1, 1}^d$). One can verify that (see \cite[Sec. IV.B]{li2019nearest}) $\nunitnbr[r] \geq  \frac{\N}{2\munit} $ for $r = \Omega(\T^{-\half})$, and $\ntimenbr[r] \geq  \frac{\T}{2\mtime} $ for $r = \Omega(\N^{-\half})$. 
    \end{example}

    \begin{example}[\examplecontinuous]
        \label{example:continuous}
        \normalfont
         In this setting, we assume $\lfun$ from \cref{assum:nonlinear} is bilinear (i.e., \cref{assum:bilinear_f} holds instead of \cref{assum:nonlinear}) and assume that $\ulf$ and $\vlf$ are generated as \iid draws (and independent of each other) from uniform distributions over a bounded set in $\real^d$ (say $[-c, c]^d$ for $c=(2/3)^{1/3}$). In this case, one can verify (see \cite[Sec. IV.B]{li2019nearest}) that $\nunitnbr[r] = \Theta(\N r^{\dhalf}) $ for $r = \Omega(\T^{-\half})$, and $\ntimenbr[r]= \Theta(\T r^{\dhalf}) $ for $r = \Omega(\N^{-\half})$ with high probability. 
    \end{example}

\subsection{Non-asymptotic guarantee}
To state our main result, we introduce some shorthands:
\begin{align}
    \label{eq:f_eta_chi}
    \errtermf &=  c(\ybnd+\nconst)^2\sqrt{\log(c'\max(\N, \T)/ \delta)}, 
\end{align}
with universal constants $c$ and $c'$. Moreover, we assume that the hyperparameter $\mbi{\threshold} =(\threshold_1, \threshold_2)$ satisfies
\begin{align}
\label{eq:eta_cond_f} 
\begin{split}
    \errtermf\parenth{(\sunit)^{-\half} +(\stime)^{-\half} } &\leq \p
     \qtext{where} 
     \\
     \thresop\defeq\threshold_1\!-\!2\sigma^2\!-\!\frac{\errtermf}{\p\sqrt{\T}},
     &\qtext{and}
     \threstp\defeq\threshold_2\!-\!2\sigma^2\!-\!\frac{\errtermf}{\p\sqrt{\N}}.
\end{split}
\end{align}
With the set-up in place, we are now ready to state our first main result.
\begin{theorem}[\anytimeboundname]
    \label{thm:anytime_bound}
    Let \cref{assum:mcar,assum:noise,assum:non_linear_f} be in force, and consider a tuple $(i, t)$.
    Given a fixed $\delta \in (0, 1)$, suppose that the hyperparameter $\mbi{\threshold}=(\threshold_1,\threshold_2)$ satisfies the regularity condition~\cref{eq:eta_cond_f}.
    Then there exists universal constants $c,c'$ such that conditional on the latent factors $\ulf, \vlf$, the doubly robust estimate  $\estobsdrs$ with sample split satisfies
    \begin{subequations}
    \label{eq:nn_bnd}
    \begin{align}
    \label{eq:nn_bnd_explicit}
        (\estobsdrs - \trueobs)^2 & \leq \biasdrs + \vardrs, 
        \qtext{where} \\
    \label{eq:bias_drs}
     \biasdrs&\defeq  32L_{H}^2\brackets{\biggparenth{\frac{\threshold_{1}'+ \frac{2\errtermf}{\p\sqrt{\T}} }{c_{f,1}}}^{4/\alpha_1}+\biggparenth{\frac{\threshold_{2}'+ \frac{2\errtermf}{\p\sqrt{\N}} }{c_{f,2} }}^{4/\alpha_2}} \qtext{and} \\
    \label{eq:var_drs}
      \vardrs &\defeq\frac{c\sigma^2}{\p}\biggparenth{ \frac{\log(\stime/\delta)}{\sunit}
    + \frac{\log(\sunit/\delta)}{\stime} + \frac{\log(\sunit\stime/\delta)}{\p^2\sunit\stime\!} },
    \end{align}
    with probability at least
    \begin{align}
    \label{eq:prob_drs}
         \psuccess \defeq 1\!-\!2\delta
        \!-\!2\sunit e^{-c'p^2\sunit}
        \!-\!2\stime e^{-c'p^2\stime}
        \!-\!e^{-c'p^3\sunit\stime}.
    \end{align}
    \end{subequations}
    \end{theorem}

    See \cref{sec:proof_of_anytime_bound} for its proof.
    \cref{thm:anytime_bound} provides a generic error bound on our doubly robust estimate for $\trueobs$ conditional on $\ulf,\vlf$, with high probability over the sampling mechanism and the noise variables, and decomposes the error into the bias term $\biasdrs$~\cref{eq:bias_drs}, and the variance term $\vardrs$~\cref{eq:var_drs}. To understand the general result above, we need to place additional structure to unpack the various terms appearing in the display. Throughout, we leverage a bilinear factor model to illustrate the gains of DR-NN over vanilla-NN.
    \begin{assumption}[Bilinear factor model]
        \label{assum:bilinear_f}
        The parameters $\Theta$ satisfy
        \begin{align}
            \trueobs\!= \angles{\lunit, \ltime}
            \qtext{for all} i \in [N+1], t\in[T+1].
        \end{align}
        Moreover, there exist constants $\uconst$ and $\vconst$ such that the latent factors satisfy $\sup_{i}\twonorm{\lunit} \leq \uconst$ and $\sup_{t}\twonorm{\ltime} \leq \vconst$, and their covariance matrices
        \begin{align}
        \label{eq:sigs}
            \Sigu  \defeq \frac{1}{\N+1} \sum_{i=1}^{\N+1} \lunit\lunit\tp
            \qtext{and}
            \Sigv \defeq \frac{1}{\T+1} \sum_{t=1}^{\T+1} \ltime\ltime\tp,
        \end{align}
         satisfy $\lamunit \defeq \lambda_{\min}(\Sigu)>0$ and  $\lamtime \defeq  \lambda_{\min}(\Sigv) >0$.
    \end{assumption}
    Our next result unpacks the DR-NN guarantee when the factor model is bilinear. Here we state a tighter bound than directly implied by \cref{thm:anytime_bound} and provide the proof in \cref{sec:proof_of_cor:bilinear}. 
    \begin{corollary}[\bilinearresultname]
    \label{cor:bilinear}
        Let \cref{assum:mcar,assum:noise,assum:bilinear_f} be in force.
    Given a fixed $\delta \in (0, 1)$, suppose that the hyperparameter $\mbi{\threshold}=(\threshold_1,\threshold_2)$ satisfies the regularity condition~\cref{eq:eta_cond_f}.
    Then there exists universal constants $c,c'$ such that conditional on the latent factors $\ulf, \vlf$, the doubly robust estimate  $\estobsdrs$ with sample split satisfies
    \begin{align}
        (\estobsdrs-\trueobs)^2 & \leq \biasdrs + \vardrs
        \qtext{where} 
     \biasdrs \defeq \frac{2}{\lamunit\lamtime}\!\biggparenth{\thresop + \frac{2\errterm}{p\sqrt{\T}}}\biggparenth{\threstp + \frac{2\errterm}{p\sqrt{\N}}},
     \label{eq:bias_bilinear_drs}
    \end{align}
    with probability at least $\psuccess$. Here $\vardrs$ and $\psuccess$ are as in \cref{thm:anytime_bound}.
    \end{corollary}

    \paragraph{Improvement for bilinear factor model and \cref{example:finite,example:continuous}}
    Before turning to a detailed discussion of how \cref{cor:bilinear} provides an improved guarantee over the vanilla NN variants, we first unpack the error via the next corollary for the settings in \cref{example:finite,example:continuous}. For clarity, we assume $\T=\N$, assume $p=\Omega(\N^{-\beta})$ for some $\beta\in[0,1]$---a larger $\beta$ implies more missingness, treat $\delta$ as a constant, and track dependency of the error only on $(\N, \T)$ in terms of $\wtil{O}$ notation that hides logarithmic dependencies, and constants that depend on $(d, \sigma^2, \uconst,\vconst,\nconst)$.

    \newcommand{\bilinexampleresultname}{Error rates for $\estobsdr$ for \lowercase{\Cref{example:finite,example:continuous}}}
    \begin{corollary}[\bilinexampleresultname]
        \label{cor:anytime_bound}
        Consider the set-up of \cref{cor:bilinear} such that $\T=\N$ and $p=\Theta(\N^{-\beta})$ and let $\errtwo \defeq \frac{\errterm}{p\sqrt{\N}}$. Then the following statements hold.
        \begin{enumerate}[label=(\alph*)]
            \item \label{item:finite_anytime} For the setting in \cref{example:finite} with $\munit=\mtime=\M$ and $\N$ large enough such that $\errtwo = o(\min_{\lunit[] \neq \lunit[]'} \norm{\lunit[]\!-\!\lunit[]'}^2_{2})$, then for the choice $\threshold_1 = \threshold_2 = 2\sigma^2 + 2\errtwo$, with probability at least $1-2\delta-2\N e^{-\N^{1-2\beta}/\M}-e^{-\N^{2-3\beta}/\M^2}$, we have
            \begin{align}
            \label{eq:finite_non_asymp_bnd}
               (\estobsdr-\trueobs)^2 =  \wtil{\order}\parenth{\frac{M}{\N^{1-\beta}} \parenth{1 + \frac{M}{\N^{1-2\beta}}}}.
            \end{align}
            \item\label{item:highd_anytime} For the setting in \cref{example:continuous} when $\beta=0$ with the choice $\threshold_1 = \threshold_2=2\sigma^2 +\errtwo+c\N^{-\frac{2}{d+4}}$, with probability at least $1-2\delta-c\N e^{-c\N^{4/(d+4)}}$, we have
            \begin{align}
             \label{eq:highd_non_asymp_bnd}
            (\estobsdr-\trueobs)^2 = 
            \wtil{\order}\sparenth{\N^{-\frac{4}{d+4}} }.
            \end{align}
        \end{enumerate}
    \end{corollary}
    See \cref{proof_of_cor:anytime_bound} for the proof, where we also provide the error expressions for part~\cref{item:highd_anytime} with a general $\beta \leq \half$. For a constant $p$ ($\beta=0$) \cref{cor:anytime_bound} provides the following error for the \drname estimate:
    \begin{align}
    \label{eq:dr_simple}
    (\estobsdr-\trueobs)^2 =
    \begin{cases}
        \wtil{\order}(\M/\N) \ \ \stext{for \cref{example:finite}, and} \\[2mm]
        \wtil{\order}(\N^{-\frac{4}{d+4}}) \stext{for \cref{example:continuous}}.
    \end{cases}
    \end{align}
    Notably $\M/\N$ resembles a parametric rate, while $\N^{-\frac{4}{d+4}}$ denotes a non-parameteric rate that suffers from the curse of dimensionality. We now discuss how in both settings, \cref{eq:dr_simple} establishes the dominance of our \drname estimate compared to the unit or time-based NN estimate. Prior work~\cite[Thm.~1]{LeeLiShahSong16}, \cite[Cor.~1]{dwivedi2022counterfactual}\footnote{\label{fn:prior_work}Both works do not provide the stated bound explicitly: \cite[Thm.~1]{LeeLiShahSong16} derives it implicitly in their proofs, and \cite[Thm.~3, Cor.~1]{dwivedi2022counterfactual} establishes the bound for an adaptively sampled setting that provides the stated bound in our setting as a special case (see \cite[Rem.~5]{dwivedi2022counterfactual}).} establishes that the error scaling for $\estobsunit$ (which is same as that for $\estobstime$ due to symmetry in assumptions) is as follows:
    \begin{align}
    \label{eq:nn_simple}
        (\estobsunit-\trueobs)^2 =
        \begin{cases}
            \wtil{\order}(\N^{-\half}+\M/\N) \  \stext{for \cref{example:finite} and} \\[2mm]
        \wtil{\order}(\N^{-\frac{2}{d+2}}) \quad \qquad \stext{for \cref{example:continuous}.}
        \end{cases}
    \end{align}
    Putting \cref{eq:dr_simple,eq:nn_simple} together, we conclude that compared to both the vanilla NN estimates, \drname estimate provides a quadratic and near-quadratic in $\N$ improvement in the error for discrete factors and continuous factors respectively.

    We note that the assumption of uniformity in \cref{example:finite,example:continuous} is for simplifying the presentation of \cref{cor:anytime_bound}, and our proofs can be easily adapted to establish error bounds that scale with the minimum value of the probability mass function instead of $\M$ with discrete factors, and with the lower bound on density for continuous factors.  With non-identity covariance matrices $\Sigu, \Sigv$, the results involve constants that scale with dimension $d$, the condition number, and the minimum eigenvalue of the matrices---so that the prior conclusions continue to hold when all these quantities are treated as constants as $\N$ and $\T$ grow. Finally, when the factor probability mass functions or density functions are not lower bounded, using arguments from  from the \cite[proof of Cor.~3]{li2019nearest} and in a manner analogous to \cite[Rem.~3]{dwivedi2022counterfactual}, we can still establish suitable lower bounds on $\nunitnbr[r]$ and $\ntimenbr[r]$ with high probability over the draw of $\lunit$ and $\ltime$ respectively; and consequently suitably modified \cref{thm:anytime_bound,cor:anytime_bound} can be interpreted as a high probability error bound on a uniformly drawn estimate from all the $\N\times \T$ estimates.

    \begin{remark}[No neighbors yields vacuous guarantee]
        \label{rem:no_nbr}
        \cref{thm:anytime_bound} applies only to the estimate~\cref{eq:obs_estimate_dr} and is meaningful when there are enough number of unit and time neighbors. If any of the problem parameters  $(\N,\T, \p)$ or the algorithm parameter $\bthreshold$ is too small, the terms $\sunit$ or $\stime$ might become $0$ so that the guarantee from \cref{thm:anytime_bound} would be vacuous.
    \end{remark}

    \subsection{Discussion about error improvement with bilinear factor model for general distributions}
    We now provide a detailed discussion on two favorable properties of the \drname compared to the unit- or time-based NN. We illustrate the key arguments using \cref{cor:bilinear} for the bilinear factor model throughout this section, and refer the reader to \cref{rem:non_lin_gen} to see how (and when) the arguments can be extended to the non-linear factor model using \cref{thm:anytime_bound}.

    We omit $(\n, \t)$ in our notation and use the shorthand:
    \begin{align}
    \label{eq:dr_decomp}
        \errdrs \defeq \biasdrs + \vardrs,
    \end{align}
    where $\biasdrs$ and $\vardrs$ respectively denote the bias and variance of the \drname estimate from \cref{cor:bilinear}.
    Prior work~\cite[Thm.~1]{LeeLiShahSong16}, \cite[Thm.~3]{dwivedi2022counterfactual} (also see \cref{fn:prior_work}) established that
    \begin{subequations}
    \label{eq:unit_time_errors}
    \begin{align}
    \label{eq:unit_error}
        (\estobsunit\!-\!\trueobs)^2 &\leq \underbrace{\frac{2}{\lamtime}\!\biggparenth{\!\thresop\!+\!\!\frac{c\errterm}{p\sqrt{\T}}\!}}_{\defeq \biasunit} + \underbrace{\frac{c'\errterm\sigma^2}{\p\sunit}}_{\defeq\varunit} \defeq  \errunit \qtext{and} \\
    \label{eq:time_error}
        (\estobstime\!-\!\trueobs)^2 &\leq \underbrace{\frac{2}{\lamunit}\!\biggparenth{\!\threstp\!+\!\!\frac{c\errterm}{p\sqrt{\N}}\!}}_{\defeq \biastime} + \underbrace{\frac{c'\errterm\sigma^2}{\p\stime}}_{\defeq\vartime} \defeq \errtime,
    \end{align}
    \end{subequations}
    with high probability. (For simplicity we treat $p$ as a constant going forward.) Comparing \cref{eq:unit_error,eq:time_error} with \cref{eq:nn_bnd}, we find that up to logarithmic factors, the bias and variance of \drname estimator satisfies
    \begin{align}
    \label{eq:drs_and_unit_time}
        \biasdrs \approx \biasunit \biastime
        \qtext{and}
        \vardrs \approx \varunit + \vartime + \trm{higher-order-term}.
    \end{align}
    In simple words, the bias and variance of $\estobsdr$ are respectively given by the product of biases and sum of variances of the two vanilla NN estimates ($\estobsunit$ and $\estobstime$).
    Note that $\biasunit$ and $\varunit$ exhibit a trade-off with $\threshold_1$: As $\threshold_1$ decreases to $2\sigma^2$ (and consequently $\thresop$ to $0$), the $\biasunit$~\cref{eq:unit_error} decreases to $\order(\T^{-\half})$, but $\varunit$ increases as $\sunit$ decreases to $0$~\cref{eq:nbr_unit}. The threshold $\threshold_2$ characterizes a similar trade-off between $\biastime$ and $\vartime$. Putting these observations together, yield two conclusions that we elaborate in the next two sections.

    Let us denote the optimal choice of $\threshold_1$ and $\threshold_2$ for minimizing $\errunit$ and $\errtime$ by  $\threshold_{\unittag}$ and $\threshold_{\timetag}$ respectively. Moreover, let the optimum choice of $\bthreshold$ for $\errdrs[\bthreshold]$ be denoted by $ \mbi{\threshold}_{\drstag} \defeq (\threshold_{1,\drstag}, \threshold_{2,\drstag})$. Note that throughout we use the fact that for any estimator, at the optimum choice of hyper-parameter, the bias and variance terms (and hence their sum, i.e., the error) are of the same order, e.g., $ \errunit[\threshold_{\unittag}] \asymp \biasunit[\threshold_{\unittag}] \asymp \varunit[\threshold_{\unittag}]$.

        \subsubsection{\drname estimate is doubly robust to heterogeneity in user and time factors} 
        \label{item:drparam} 
        First, we show that the \drname estimate is (roughly) at least as good as the better of the two vanilla NN estimates.

        Suppose $\estobstime[\threshold_{\timetag}]$ provides a vanishing error, so that at optimum we have $\errtime[\threshold_{\timetag}] \asymp \biastime[\threshold_{\timetag}] \asymp \vartime[\threshold_{\timetag}] = \Omega(\N^{-1/2})$. 
        Next, choose $\threshold_1$ large enough, such that $\biasunit = \Theta(1)$ and $\varunit = \otil(\N\inv)$ (since $\sunit = \N$ as $\threshold_1\to\infty$) so that 
        \begin{align}
        \label{eq:one_good}
            \errdrs[\mbi{\threshold}_{\drtag}] \sless{(i)} \errdrs[(\threshold_1, \threshold_{\timetag})] \precsim \biasunit \biastime[\threshold_{\timetag}] + \varunit + \vartime[\threshold_{\timetag}]
            \stackrel{(ii)}{\asymp} \biastime[\threshold_{\timetag}] + \vartime[\threshold_{\timetag}]
            = \errtime[\threshold_{\timetag}],
        \end{align}
        where step~(i) follows from optimality of $\mbi{\threshold}_{\drtag}$ and step~(ii) from the fact that $\vartime[\threshold_{\timetag}]=\Omega(\N^{-1/2})$ and $\varunit[\threshold_{\unittag}] = \otil(\N\inv)$. 
        Flipping the roles of $\estobsunit$ and $\estobstime$, we can find $\threshold_2$ such that
        \begin{align}
        \label{eq:one_good_time}
            \errdrs[\mbi{\threshold}_{\drtag}] \leq \errdrs[(\threshold_{\unittag}, \threshold_2)] \precsim \biasunit[\threshold_{\unittag}] \biastime + \varunit[\threshold_{\unittag}] + \vartime
            \stackrel{(i)}{\asymp} \biasunit[\threshold_{\unittag}] + \varunit[\unittag]
            = \errunit[\threshold_{\unittag}].
        \end{align}
        Putting the pieces together, we have concluded that
        \begin{align}
        \label{eq:best_of_both}
            \errdrs[\mbi{\threshold}_{\drtag}] \precsim \min\braces{\errunit[\threshold_{\unittag}],\, \errtime[\threshold_{\timetag}]}.
        \end{align}
        In other words, the error guarantee of the \drname estimate (after tuning $\bthreshold$) is at least as good as that obtained by tuning either of the two vanilla NN estimates. And thus we conclude that $\estobsdrs$ is doubly robust to the heterogeneity in the user and time factors---as long as one of them provides enough number of good neighbors (either for unit $\n$ or time $\t$), our DR estimator with well tuned hyperparameters provides a non-trivial error guarantee. 
        \subsubsection{\drname estimate improves upon both unit and time neighbors estimate when each provides vanishing errors}
        \label{item:improvement}
         We now consider the setting where both time and unit neighbors estimate provides vanishing error. 
         Note that for the optimal choices of thresholds, 
        \begin{align}
        \label{eq:scaling}
            \errunit[\threshold_{\unittag}] \asymp \biasunit[\threshold_{\unittag}] \asymp \varunit[\threshold_{\unittag}] \ll 1,
            \quad
            \errtime[\threshold_{\timetag}] \asymp 
            \biastime[\threshold_{\timetag}] \asymp  \vartime[\threshold_{\timetag}] \ll 1.
        \end{align}
        These rough calculations immediately imply that the choice $\mbi{\threshold}= (\threshold_{\unittag},\threshold_{\timetag})$ is sub-optimal for $\estobsdrs$, since the error for this estimate at such a choice is dominated by the variance term:
        \begin{align}
        \label{eq:threshold_choice_drs}
            \biasunit[\threshold_{\unittag}]\biastime[\threshold_{\timetag}]
            \asymp \varunit[\threshold_{\unittag}] \vartime[\threshold_{\timetag}] \ll \varunit[\threshold_{\unittag}] + \vartime[\threshold_{\timetag}]
            \stackrel{\cref{eq:dr_decomp}}{\Longrightarrow}
            \errdrs[(\threshold_{\unittag},\threshold_{\timetag})]
            = \order\parenth{\varunit[\threshold_{\unittag}] + \vartime[\threshold_{\timetag}]}.
        \end{align}
        In contrast, the optimal choices would satisfy $ \biasunit\biastime \asymp \varunit + \vartime$. Under regularity condition that 
        as we increase the value of $(\threshold_1, \threshold_2)$ (in each coordinate) from $\threshold_{\unittag},\threshold_{\timetag}$ the bias terms in \cref{eq:threshold_choice_drs} increase, and the variance terms decrease, we find that at the optimum choice for \drname estimate---denoted by $ \mbi{\threshold}_{\drstag} \defeq (\threshold_{1,\drstag}, \threshold_{2,\drstag})$, we have $\varunit[\threshold_{1,\drstag}] \ll \varunit[\threshold_{\unittag}]$, and $\vartime[\threshold_{2,\drstag}] \ll \vartime[\threshold_{\timetag}]$,
        and consequently that 
        \begin{align}
        \label{eq:err_improv}
            \errdrs[\mbi{\threshold}_{\drstag}] \!\asymp\! \biasunit[\threshold_{1,\drstag}]\biastime[\threshold_{2,\drstag}] \!\asymp\! \varunit[\threshold_{1,\drstag}]\!+\!\vartime[\threshold_{2,\drstag}] 
            \!\ll\! \min\braces{\errunit[\threshold_{\unittag}], \errtime[\threshold_{\timetag}]}
            \!\asymp\! \min\sbraces{ \varunit[\threshold_{\unittag}], \vartime[\threshold_{\timetag}]}.
        \end{align}
        That is, when both hyperparameters for \drname estimate are tuned, its error would provide a significant improvement over the error obtained by the either of the unit or time neighbors estimate.

    \subsection{Improvement for non-linear factor model}
    \label{rem:non_lin_gen}
        For a non-linear factor model satisfying \cref{assum:non_linear_f}, the general guarantee of \cref{thm:anytime_bound} guarantees the following decomposition (rather than \cref{eq:drs_and_unit_time})
        \begin{align}
        \label{eq:bias_drs_f}
             \biasdrs \approx (\biasunit)^2 + (\biastime)^2
        \qtext{and}
        \vardrs \approx \varunit + \vartime + \trm{higher-order-term}.
         \end{align} 
         When $\estobsunit$ and $\estobstime$ have vanishing error, the choice $\mbi{\threshold}= (\threshold_{\unittag},\threshold_{\timetag})$ is sub-optimal for $\estobsdrs$, since
        \begin{align}
        \label{eq:threshold_choice_drs_nonlin}
            \biasdrs[(\threshold_{\unittag},\threshold_{\timetag})] \approx (\biasunit[\threshold_{\unittag}])^2 + (\biastime[\threshold_{\timetag}])^2 \ll
            \biasunit[\threshold_{\unittag}] + \biastime[\threshold_{\timetag}]
            \asymp  \varunit[\threshold_{\unittag}] + \vartime[\threshold_{\timetag}] \approx \vardrs[(\threshold_{\unittag},\threshold_{\timetag})]. 
        \end{align}
        Under regularity conditions analogous to that discussed between displays~\cref{eq:threshold_choice_drs,eq:err_improv}, for the optimal choice $\bthreshold_{\drtag}$, we would have $\threshold_{1,\drstag} -2\sigma^2 \gg \threshold_{\unittag}-2\sigma^2$ and  $\threshold_{2,\drstag} -2\sigma^2 \gg \threshold_{\timetag}-2\sigma^2$ so that $\varunit[\threshold_{1,\drstag}] \ll \varunit[\threshold_{\unittag}]$, $\vartime[\threshold_{2,\drstag}] \ll \vartime[\threshold_{\timetag}]$ and hence
        \begin{align}
            \errdrs[\mbi{\threshold}_{\drstag}] \!\asymp\! \varunit[\threshold_{1,\drstag}]\!+\!\vartime[\threshold_{2,\drstag}] \ll \errunit[\threshold_{\unittag}] + \errtime[\threshold_{\timetag}].
        \end{align}
        In full generality, this display does not provide a guarantee as strong as \cref{eq:best_of_both} or \cref{eq:err_improv} (and its due to the sum of squared biases in \cref{eq:bias_drs_f}). However, under mild conditions, \cref{thm:anytime_bound} does establish that $\estobsdrs$ provides a significant improvement over both the vanilla NN estimates as our next corollary illustrates (see \cref{sec:proof_of_cor:dr_gen_improv} for the proof):
    \newcommand{\improvcorresultname}{Error rates for $\estobsdr$ with non-linear factor models}
    \begin{corollary}[\improvcorresultname]
    \label{cor:dr_gen_improv}
        Let $p \geq c$ for some constant $c>0$ and enforce the setting of \cref{thm:anytime_bound}. If $\sunit[r] = \Theta(N r^{\beta_1})$ and $\stime[r] = \Theta(\T r^{\beta_2})$ for some $\beta_1, \beta_2 \geq 0$, then there exists a choice of $\bthreshold=(\eta_1,\eta_2)$ such that
        \begin{align}
            (\estobsdrs - \trueobs)^2 = \otil_P(\N^{-(\frac{4}{4+\alpha_1\beta_1} \wedge \frac{2}{\alpha_2}) } + \T^{-(\frac{4}{4+\alpha_2\beta_2} \wedge \frac{2}{\alpha_1}) }),
        \end{align}
        where $\otil$ hides logarithmic dependence on $(\N, \T)$. Consequently, if $\N=\T$ and $\alpha_1=\alpha_2=\alpha$ and $\beta_1=\beta_2=\beta$, then with constant probability
        \begin{align}
        \label{eq:dr_nn_general_rate_simplified}
            (\estobsdrs - \trueobs)^2 = \otil(\N^{-(\frac{4}{4+\alpha\beta} \wedge \frac 2{\alpha}) }).
        \end{align}
    \end{corollary}
    
    For the unit-NN or time-NN estimate, the analog of the error bound~\cref{eq:dr_nn_general_rate_simplified} (under the same assumptions) can be derived to scale as follows:
    \begin{align}
    \label{eq:unit_nn_general_rate_simplified}
        (\estobsunit-\trueobs)^2= \otil(\N^{-(\frac{2}{2+\alpha\beta}\wedge \frac 1{\alpha}) })
        \begin{cases}
            \stackrel{(i)}{\succsim} \otil(\N^{-(\frac{4}{4+\alpha\beta} \wedge \frac 2{\alpha}) }) \\ 
            \stackrel{(ii)}{\gg} \otil(\N^{-(\frac{4}{4+\alpha\beta} \wedge \frac 2{\alpha}) }) \stext{if} \alpha\beta>0 \stext{or} \alpha>2.
        \end{cases}
    \end{align}
    That is, up to logarithmic factors, the upper bound~\cref{eq:dr_nn_general_rate_simplified} on the error of DR-NN estimator is \emph{never worse} than the upper bound~\cref{eq:unit_nn_general_rate_simplified} on the error of unit-NN (or time-NN) estimator; and is \emph{strictly better} if either $\alpha\beta>0$ or $\alpha > 2$. Here steps~(i) and (ii) are consequences of the following two observations: First,
    \begin{align}
        a_j\leq b_j \stext{for} j=1, 2 \implies \min(a_1, a_2) \leq \min(b_1, b_2)
    \end{align}
    (where $\leq$ can be replaced by $<$ as well; and second,
    \begin{align}
        \frac{2}{2+\alpha\beta} \leq \frac{4}{4+\alpha\beta} \stext{for all} \alpha, \beta \geq 0
        \qtext{and}
        \frac{2}{2+\alpha\beta} < \frac{4}{4+\alpha\beta} \stext{if}  \alpha\beta>0.
    \end{align}

\subsection{Asymptotic guarantees for the DR-NN estimates}
We now establish asymptotic guarantees for the estimate $\estobsdr$ under an appropriate scaling of the problem parameters $(p, \N, \T)$ assuming that the hyper-parameter $\mbi{\threshold}$ is tuned suitably.

    \newcommand{\asympresultname}{Asymptotic error guarantee for $\estobsdr$ with non-linear factor models}
    \begin{theorem}[\asympresultname]
    \label{thm:asymp}
    Let the set-up of \cref{thm:anytime_bound} hold with constant $L_{H}, c_{f, 1}$, and $c_{f, 2}$ and define
    \begin{align}
    \label{eq:asymp_bias}
        \biasdrsasymp = \brackets{\bigparenth{\threshold_{1}\!-\!2\sigma^2 \!+\! \frac{c\sqrt{\log(\N+\T)}}{\p\sqrt{\T}}}^{4/\alpha_1}\  +\ \bigparenth{\threshold_{2}\!-\!2\sigma^2 \!+\! \frac{c'\sqrt{\log(\N+\T)}}{\p\sqrt{\N} }}^{4/\alpha_2} }.
    \end{align}
    Then conditional on the latent factors $\ulf = (\lunit[j])_{j=1}^{\infty}, \vlf=(\ltime[\t'])_{\t'=1}^{\infty}$, the following results hold as both $\N$ and $\T \to \infty$ and $\bthreshold \to (2\sigma^2, 2\sigma^2)$.
    \begin{enumerate}[label=(\alph*),leftmargin=*]
        \item\label{item:cons} \tbf{Asymptotic consistency:} If $\biasdrsasymp\to0$ and
        \begin{align}
        \label{eq:cons_var_p}
            \frac{p\log(\stime)\!+\!\log(\sunit)}{p^2\sunit} \!+\! \frac{p\log(\sunit)\!+\!\log(\stime)}{p^2\stime} \!+\! \frac{p\log(\sunit\stime) \!+\! 1}{p^3\sunit\stime}  &\to 0,
        \end{align}
         then $\estobsdr \sinprob \trueobs$.
        \item\label{item:clt}\tbf{Asymptotic normality}: Recall the definition~\cref{def:n_it} of $\effnn$. If 
        \begin{align}
        \label{eq:clt_bias}
            \effnn\biasdrsasymp  &\to 0 \\
            \label{eq:clt_p_cond}
            \qtext{and} 
            \min\braces{\frac{\log(\sunit)}{p^2\sunit},\  \frac{\log(\stime)}{p^2\stime}, \frac{1}{p^3\sunit\stime} } &\to 0,
        \end{align}
        then $ \sqrt{\effnn}(\estobsdr- \trueobs)  \Longrightarrow \mc N(0, \sigma^2)$.
    \end{enumerate}
    \end{theorem}

    \begin{remark}[Bilinear factor model]
    When \cref{assum:bilinear_f} also holds in \cref{thm:asymp} with the latent factors $\mc U, \mc V$ satisfying $\lim_{\N\to\infty} \frac{\sum_{\n=1}^{\N+1} \lunit\lunit\tp}{\N+1} \succeq c_1 I $ and $\lim_{\T\to\infty} \frac{\sum_{\t=1}^{\T+1} \ltime\ltime\tp}{\N+1}\succeq c_2 I$ for some constants $c_1$ and $c_2$, then the definition~\cref{eq:asymp_bias} of $\biasdrsasymp$ can be sharpened to
    \begin{align}
    \biasdrsasymp \defeq 
        \! \frac{2}{\lamunit\lamtime}\bigparenth{\threshold_1\!-\!2\sigma^2 + \frac{c\sqrt{\log(\N+\T)}}{p\sqrt{\T}}}\bigparenth{\threshold_2\!-\!2\sigma^2  + \frac{c'\sqrt{\log(\N+\T)}}{p\sqrt{\N}}}.
    \end{align}
    The proof of this claim can be derived directly from using \cref{cor:bilinear} instead of \cref{thm:anytime_bound} in the proof of \cref{thm:asymp}.
    \end{remark}

     \cref{thm:asymp}\cref{item:clt} states the conditions under which the interval~\cref{eq:unit_ci} provides a valid $1\!-\!\alpha$ coverage asymptotically. We can deduce from the proof that $\effnn \asymp\frac{p\stime\sunit}{\stime+\sunit}$, so that for $\bthreshold=\bthreshold_{\drtag}$, the width of the asymptotic confidence interval~\cref{eq:unit_ci}, denoted as $\mrm{Width}_{\drtag}$, scales as
    \begin{align}
        \mrm{Width}_{\drtag} \asymp \frac{1}{\sqrt{p\min(\sunit[\threshold_{1,\drstag}'], \stime[\threshold_{2,\drstag}'])}} \asymp \sqrt{\max\sbraces{\varunit[\threshold_{1,\drstag}],\vartime[\threshold_{2,\drstag}]}}
    \end{align}
    On the other hand, for the vanilla NN estimates, the width of the corresponding asymptotic confidence intervals (when the estimates are asymptotically normal) scale as 
    \begin{align}
       \mrm{Width}_{\unittag} \asymp \sqrt{\varunit[\threshold_{\unittag}]}
       \qtext{and}
       \mrm{Width}_{\timetag} \asymp \sqrt{\vartime[\threshold_{\timetag}]}
    \end{align}
    When the underlying model is a bilinear factor model, combining the observations with our discussion in \cref{item:improvement}, we find that
    \begin{align}
        \mrm{Width}_{\drtag} \asymp \sqrt{\max\sbraces{\varunit[\threshold_{1,\drstag}],\vartime[\threshold_{2,\drstag}]}}
        &\leq \sqrt{\varunit[\threshold_{1,\drstag}]+\vartime[\threshold_{2,\drstag}]} \\ 
        &\stackrel{\cref{eq:err_improv}}{\ll} \sqrt{\min\sbraces{\varunit[\threshold_{\unittag}], \vartime[\timetag]} }\\ 
        &\asymp \min\sbraces{\mrm{Width}_{\unittag}, \mrm{Width}_{\timetag}}.
    \end{align}
    Thus, for the settings considered in \cref{cor:anytime_bound}, under suitable scaling conditions with a constant $p$, we can deduce that the width of asymptotic confidence interval for $\estobsdr$ scales as $\N^{-\half+\delta}$ for \cref{example:finite} with a constant $M$, and as $\N^{-\frac{2}{d+4}+\delta}$ for \cref{example:continuous}, where $\delta>0$ denotes an arbitrary small constant. In contrast, for the vanilla NN variants, under the same settings, the widths of the intervals scale as $\N^{-\quarter+\delta}$ and $\N^{-\frac{1}{d+2}}$ for \cref{example:finite,example:continuous} respectively (see discussion after \cite[Cor.~2]{dwivedi2022counterfactual}), thereby illustrating the near-quadratic improvement offered by \drname procedure for the asymptotic confidence interval width.

    \begin{remark}[Growth of $\effnn$]
        The condition~\cref{eq:clt_bias} requires that $\effnn$ does not grow too quickly compared to the decay of the bias. Such a condition is natural for nearest neighbor like estimators and has also arisen in prior work~\citep[Thm.~3.2(b)]{dwivedi2022counterfactual}. Such an issue can be tackled by random sub-sampling the neighbors and limit the number of neighbors being used to construct the interval so that the condition \cref{eq:clt_bias} holds by design (e.g., for $\estobsunit$, \cite{dwivedi2022counterfactual} randomly sub-sample the units from $\sunit$). 
    \end{remark}

\section{Experiments}
\label{sec:experiments}
To accompany our theoretical results, we demonstrate the empirical performance of DR-NN on simulated data and conduct a case study on data from a mobile health clinical trial.
\subsection{Simulations}
\label{sub:simulations}
\rd{Some of these will be pushed to appendix}
We generate data using three different models with $d$-dimensional unit latent factors $\ulf \defeq \sbraces{\lunit, \n \in[\N]}$ and time latent factors $\vlf \defeq \sbraces{\ltime, \t\in[\T]}$ drawn i.i.d from uniform distributions in $\real^d$ and Gaussian noise as follows:
\begin{talign}
\label{eq:sim_lf}
\lunit \sim \text{Unif}[-1, 1]^d, \ltime \sim \text {Unif}[-1, 1]^d, \text{ and } \noiseobs \sim \mathcal{N}(0, \sigma_\epsilon^2).
\end{talign}
For our first model, we mimic the setting of \cref{example:continuous} with bilinear $f$:
\begin{align}
\label{eq:sim_model1}
\trueobs = \angles{\lunit, \ltime} \qtext{for} \n \in[N], \t\in[T].
\end{align}
In our second model, we introduce a nonlinear $f$, which we define as the expit function:
\begin{align}
\label{eq:sim_model2}
\trueobs = \mrm {expit}(\angles{\lunit, \ltime}) \qtext{for} \n \in[N], \t\in[T].
\end{align}
Our third model introduces another nonlinear $f$, the hyperbolic sin function:
\begin{align}
\label{eq:sim_model3}
\trueobs = \mrm {sinh}(\angles{\lunit, \ltime}) \qtext{for} \n \in[N], \t\in[T].
\end{align}
The missingness pattern of the generated matrix $\Theta$ is MCAR with distribution Bernoulli$(p)$. (see $\cref{assum:mcar}$). For all experiments, we set $p = 0.5, d = 4, \text{ and } \sigma_\epsilon = 0.001$. When estimating some entry $\theta_{i, t}$, we employ the 2 x 2 sample split described in \cref{rem:data_split}. We split the time points into two sets: $\mbf T_1 \defeq \{t\in 1,\dots, T/2\}$, $\mbf T_2 \defeq \{t\in T/2+1,\dots, T\}$. We split the units in a similar fashion: $\mbf N_1 \defeq \{j\in 1,\dots, N/2\}$, $\mbf N_2 \defeq \{n\in N/2+1,\dots, N\}$. Then, we construct neighborhood $\estnbr$ using time points $\mbf T_1$ to compute unit distances $\rhounit$ for $j \in \mbf N_2$. Similarly, we construct neighborhood $\testnbr$ using units $\mbf N_1$ to compute time distance $\rhotime$ for $t' \in \mbf T_2$. Although our partitioning strategy does not choose the the split randomly as in \cref{rem:data_split}, it is functionally equivalent as the unit and time latent factors are drawn uniformly at random in our simulations. We further examine the effect of the sample split in \cref{sub:experiments_data_split}
\subsubsection{Performance on bilinear factor model}
\label{sub:blin_perf}

\cref{fig:blin_f_uv11} depicts the performance of DR-NN for a bilinear factor model $\cref{eq:sim_model1}$ and compares it to user-NN, time-NN, and USVT, a baseline matrix completion algorithm \cite{Chatterjee15}. In panel (a) of \cref{fig:blin_f_uv11}, we show the absolute error $|\estobs\!-\!\trueobs|$ averaged over 30 trials for $t = T$ and $A_{i, t} = 1$ as $T$ grows larger (for ease of analysis, $N =T$). We see that the error decay rate for DR-NN is roughly in line with our theoretical error rate of $\wtil{\order}(T^{0.5})$ from \cref{eq:dr_simple} and that DR-NN exhibits noticeably better decay rates than user-NN, time-NN, and USVT. In panel (b), we illustrate the distribution of errors $|\estobs\!-\!\trueobs|$ for each method across a single trial at $N = T = 128$. We see significant error improvement for DR-NN compared to the other methods.

\begin{figure}
  \centering
  \begin{tabular}[b]{c}
    \includegraphics[width=0.44\linewidth]{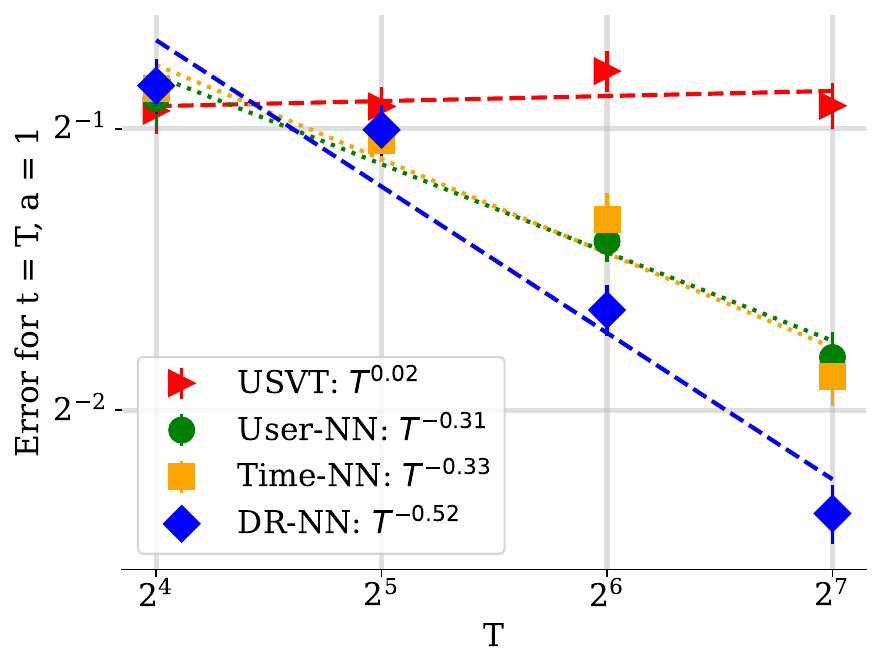} \\
    \small (a) Average absolute error decay across users
  \end{tabular} \qquad
  \begin{tabular}[b]{c}
    \includegraphics[width=.44\linewidth]{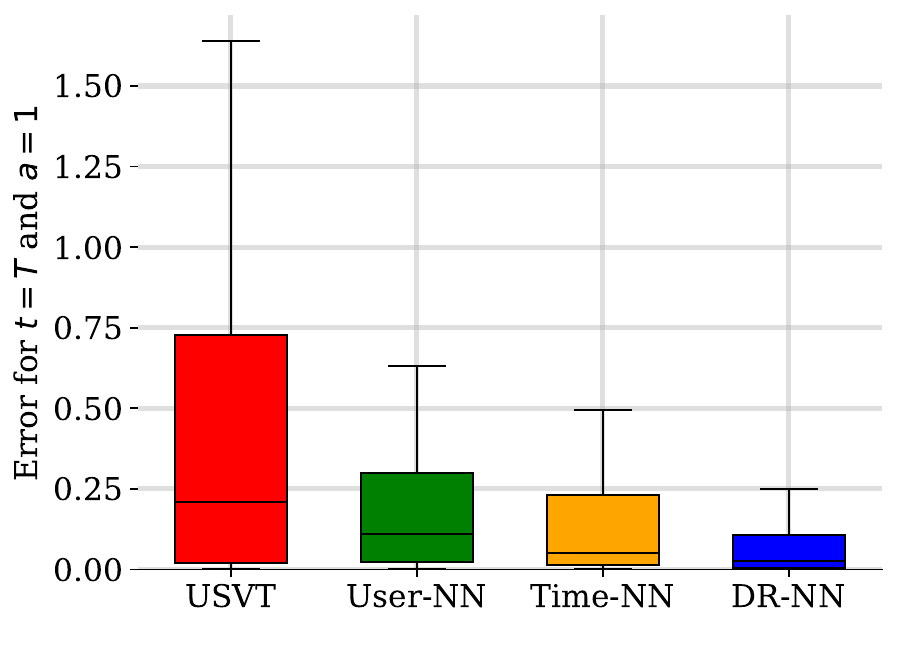} \\
    \small (b) Variation of error across users (T = 128)
  \end{tabular}
  \caption{\textbf{DR-NN absolute error across users with bilinear $f$}. We plot the performance of DR-NN against user-NN, time-NN and USVT \cite{Chatterjee15}. In panel (a), we demonstrate the absolute error decay across the observed entries of column $t = T$ with $N = T$ averaged over 30 trials. In panel (b), we show the variation of absolute error of a single trial with $N = T =128$. }
  \label{fig:blin_f_uv11}
\end{figure}

\subsubsection{Performance on nonlinear factor models}
In \cref{fig:nonlin_f_uv11}, we compare absolute error decay rates for DR-NN using the nonlinear factor models \cref{eq:sim_model2} and \cref{eq:sim_model3}. As in \cref{sub:blin_perf}, we employ the 2x2 data split and estimate entries $\estobs$ where $t = T$ and $A_{i, t} = 1$. Although our theoretical guarantees are not as strong for nonlinear factor models, our simulation results indicate that DR-NN can still offer noticeable error improvements over baseline methods. In panel (a) of \cref{fig:nonlin_f_uv11}, we show that, under an $\mrm{expit}$ nonlinear factor model, DR-NN achieves error decay rates that significantly improve upon user-NN, time-NN and USVT. In panel (b), under the sinh nonlinear factor model, DR-NN has error decay rates that are at least as good as the best of user-NN and time-NN. Add discussion about diff between sinh and expit

\begin{figure}
  \centering
  \begin{tabular}[b]{c}
    \includegraphics[width=0.44\linewidth]{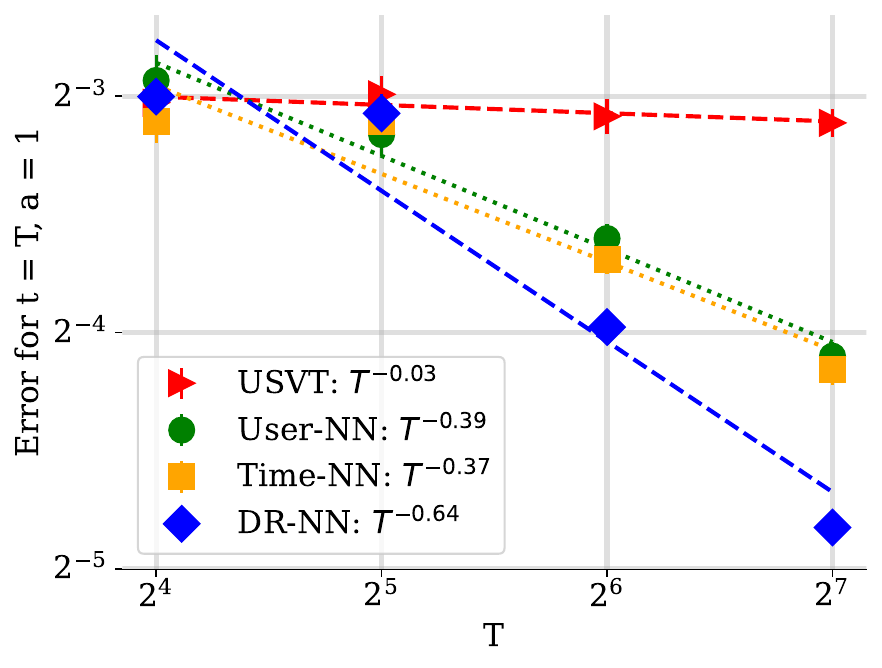} \\
    \small (a) Average error across users with expit $f$
  \end{tabular} \qquad
  \begin{tabular}[b]{c}
    \includegraphics[width=.44\linewidth]{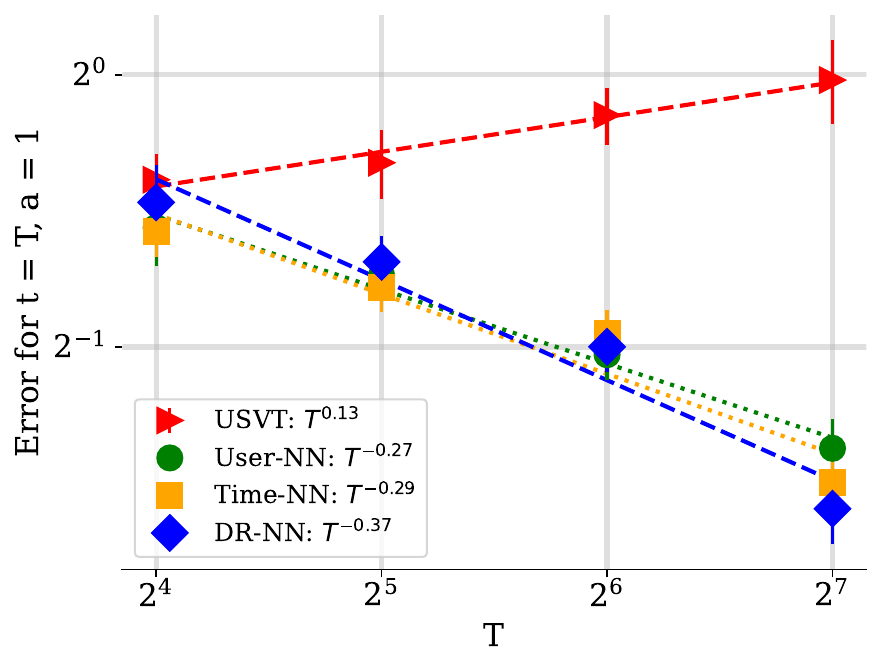} \\
    \small (b) Average error across users with sinh $f$
  \end{tabular}
  \caption{\textbf{DR-NN absolute error across users with nonlinear $f$}. We plot the performance of DR-NN against user-NN, time-NN and USVT \cite{Chatterjee15}. In panel (a), we demonstrate the absolute error decay across the observed entries of column $t = T$ with $N = T$ averaged over 30 trials using an expit nonlinear factor model. In panel (b), we depict the absolute error decay in the same setting except with a sinh nonlinear factor model. }
  \label{fig:nonlin_f_uv11}
\end{figure}

\subsubsection{Effect of no sample-split}
\label{sub:experiments_data_split}
In this section, we investigate the effect of the sample split proposed in \cref{rem:data_split} for use in the theoretical analysis. The motivation behind the sample split was to eliminate data re-use when computing the neighbor sets $\estnbr$ and $\testnbr$ that might have led to hard-to-analyze bias. However, the sample split inherently increases the variance of the estimate as it restricts the number of entries over which the neighborhoods are constructed. We depict this bias-variance tradeoff in \cref{fig:ssplit_exp}. We estimate the entry $\theta_{1, 1}$ for 30 trials with and without a sample split across increasing matrix size with $N = T$ and bilinear and nonlinear $f$ as in \cref{eq:sim_model1} and \cref{eq:sim_model3} respectively. 

\cref{fig:ssplit_exp} demonstrates that the error $|\hat{\theta}_{1, 1} - \theta_{1, 1}|$ is consistently lower when constructing $\estnbr$ and $\testnbr$ without the sample split than with the sample split under both bilinear (panel (a)) and nonlinear (panel (b)) models. The reduction in error indicates the decrease in variance without the sample split is comparatively larger than the bias reduction under the sample split. The disparity between the two methods is dampened (but not erased) with the nonlinear factor model, which is indicative of higher bias due to data reuse than under the bilinear model. In practice, it may be more effective not to employ a sample split to limit the variance of the estimates.

\begin{figure}
  \centering
  \begin{tabular}[b]{c}
    \includegraphics[width=0.44\linewidth]{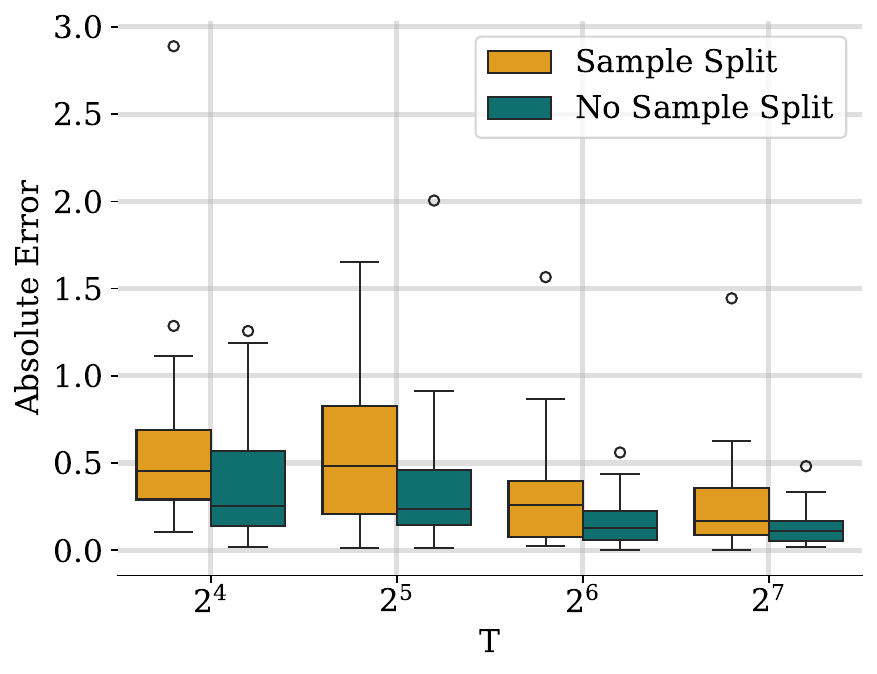} \\
    \small (a) Bilinear factor model
  \end{tabular} \qquad
  \begin{tabular}[b]{c}
    \includegraphics[width=0.44\linewidth]{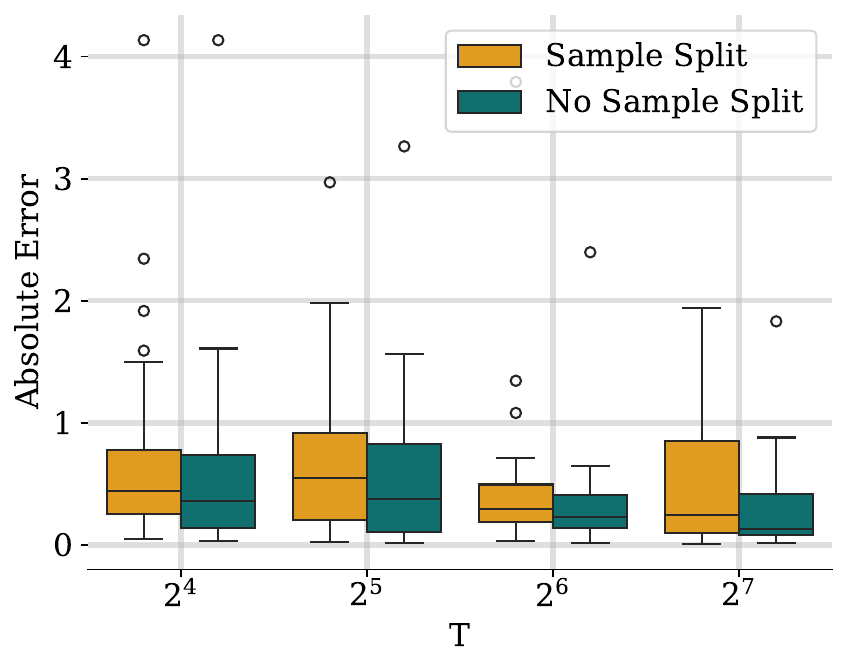 } \\
    \small (b) Sinh nonlinear factor model
  \end{tabular} 
  \caption{\textbf{DR-NN absolute error distribution across 30 trials with and without sample split.} We plot the absolute error $|\hat{\theta}_{1, 1} - \theta_{1, 1}|$ with estimates constructed by DR-NN across 30 trials and increasing matrix size with $N = T$. The orange bars represent the distribution of errors when a sample split was used to create $\testnbr$ and $\estnbr$. The teal bars represent the distribution of errors when the entire matrix could be used to construct the estimate (i.e. no sample split). Panel (a) depicts the performance under a bilinear factor model and panel(b) depicts performance under a nonlinear factor model with hyperbolic sin.}
  \label{fig:ssplit_exp}
\end{figure}

\subsection{Application in real world data}
\label{sub:exp_rwd}We demonstrate the performance of DR-NN on data from the HeartSteps V1 study \cite{klasnja2019efficacy}, a clinical trial for a mobile health intervention designed to encourage anti-sedentary activity (walking, stretching, etc.). The $N = 37$ participants in the trial were given the HeartSteps mobile app and a Jawbone activity tracker. The app sent out notifications to users at various times (decision points) during the day. Participants in the trial received notifications independently with probability $p = 0.6$ for 5 pre-determined decision points per day for 40 days~($T = 200$). We construct a data matrix $Y_{i, t}$ as described in \cref{eq:model_mc_causal}. 

\begin{talign}
\label{eq:model_mc_causal}
    \obs^{(a)} = \trueobsvar_{\n, \t}^{(a)} + \noiseobs
    \qtext{for} a\in [0, 1]
    \quad \qtext{and}
    \obs = \obs^{(\miss)},
\end{talign}
where $Y_{i, t}$ represents the observed one hour step count of participant $i$ after decision point $t$. Treatment $a = 1$ indicates that a notification was sent and $a = 0$ indicates that a notification was not sent. 
\subsubsection{Results}
To test the performance of DR-NN, we first randomly hold out 20\% of entries with $a = 1$ (n = 297) and use the remaining entries under $a = 1$ (n = 1485) to estimate the held-out entries. We repeat the same process with entries under $a = 0$ (held out 845 out of 4226). Given the findings in \cref{sub:experiments_data_split}, we do not employ a sample split for our estimates. 
\begin{figure}
  \centering
  \begin{tabular}[b]{c}
    \includegraphics[width=0.44\linewidth]{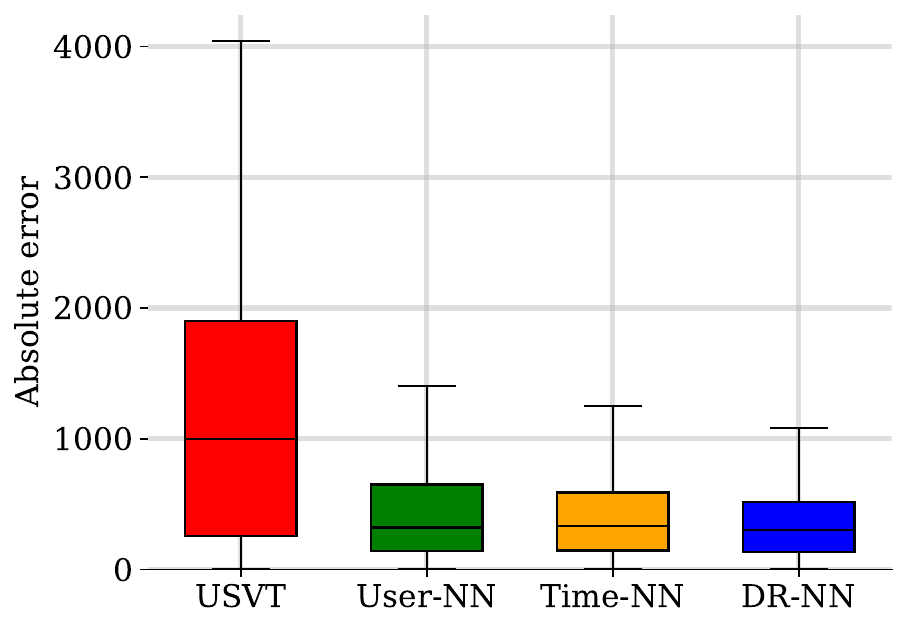} \\
    \small (a) Variation of error for a = 1
  \end{tabular} \qquad
  \begin{tabular}[b]{c}
    \includegraphics[width=0.44\linewidth]{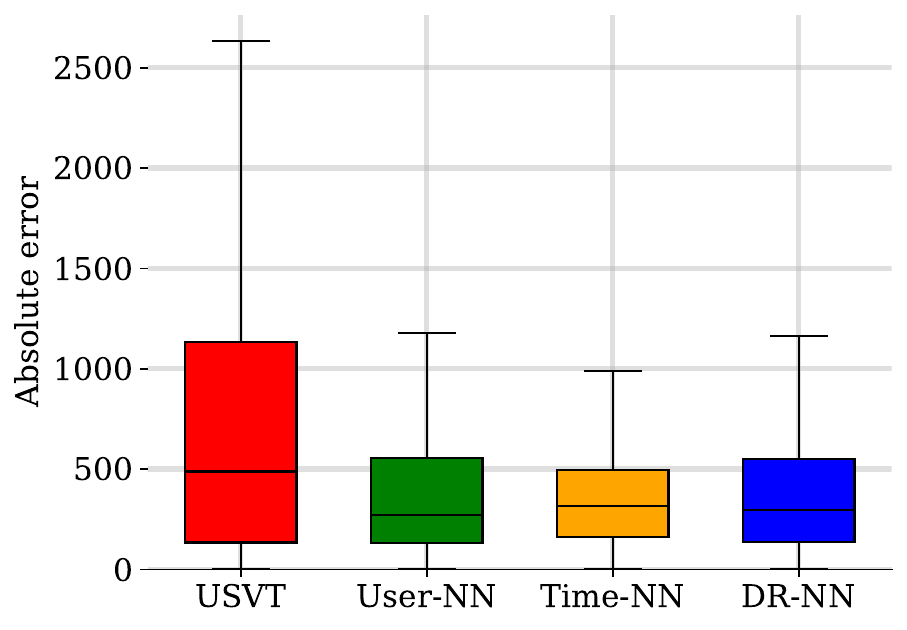} \\
    \small (b) Variation of error for a = 0
  \end{tabular} 
  \caption{\textbf{DR-NN absolute error distribution on HeartSteps v1 data.} We plot the distribution of absolute error for DR-NN and baselines user-NN, time-NN, and USVT \cite{Chatterjee15}. In panel (a), we show the error of estimates of entries that received a notification (a = 1) and in panel (b), we show the error of estimates that did not receive a notification (a = 0). }
  \label{fig:heartsteps_exp}
\end{figure} 
In \cref{fig:heartsteps_exp}, we plot the distribution of absolute error for DR-NN and the baseline methods user-NN, time-NN, and USVT \cite{Chatterjee15}. In line with our simulation findings, DR-NN demonstrates performance on the HeartSteps data that is at least as good as user-NN and time-NN and significantly improves upon USVT for entries under any treatment.

    \section{Discussion}
    \label{sec:possible_extensions}

    In this section, we present a detailed discussion of our results and its consequences in the context of broader literature on missing data, doubly robust estimation, and causal panel data problems. We start by illustrating in \cref{sub:tensor} how the intuitive construction from \cref{sec:algorithm} can guide us to derive a ``triply robust''-NN for a tensor completion problem. Next, in \cref{sec:alter_derive} we show how alternate principles from related work can be used to derive the same DR-NN estimate. \cref{sub:consequence_causal} presents a discussion of the consequences of our results more broadly for problems in counterfactual inference and causal panel data.

    \subsection{``Triply robust'' nearest neighbors for tensor denoising}
    \label{sub:tensor}
    The ideas considered in this work can also be extended to construct improved nearest neighbor estimates for denoising tensor latent factor model. Consider an $\N\times\T \times M$ dimensional tensor $Y$ that satisfies
    \begin{align}
    Y = \Theta + E
    \qtext{or equivalently}
        Y_{i, t, a} = 
            \theta_{i, t, a} + \noiseobs[i, t, a]
        \qtext{where}
        \theta_{i, t, a} = \sum_{r=1}^{d} u_{i, r} v_{t, r} w_{a, r},
    \end{align}
    i.e., $\theta_{i, t, a} = u_i \bigotimes v_t \bigotimes w_{a}$ for $(i, t, a) \in [N] \times [\T] \times [M]$. Here $u_i, v_t, w_a \in \real^d$ denote the latent factors associated with the three indices and all dimensions $\N, \T, M$ are assumed to be large.
    For simplicity, we assume the tensor is fully observed (so that we are doing tensor denoising) but our estimate can be easily generalized to the setting where entries in the tensor are missing completely at random (so that we are doing tensor completion). 

    Consider the task of estimating $\theta_{i, t, a}$ via nearest neighbors and suppose we are given the neighbors along the three dimensions: (i) unit neighbors $\estnbrnotag^{\trm{unit}}_{i, t, a} \subset [\N]$ (ii) time neighbors $\estnbrnotag^{\trm{intervention}}_{i, t, a} \subset [\T] $ and (iii) $\estnbrnotag^{\trm{unit}}_{i, t, a} \subset [M]$. Then the 
    nanilla nearest neighbor estimates for estimating $\theta_{i, t, a}$ would correspond to averaging the outcomes (i) $\sbraces{Y_{j, t, a}, j\in\estnbrnotag^{\trm{unit}}_{i, t, a}}$, namely unit-NN, (ii) $\sbraces{Y_{i, t', a}, t'\in\estnbrnotag^{\trm{time}}_{i, t, a}}$, namely time-NN, and (iii) $\sbraces{Y_{i, t, a'}, a'\in\estnbrnotag^{\trm{intervention}}_{i, t, a}}$, namely intervention-NN. 

    The ideas in the paper can be directly applied to build an estimator that provides good error as long as any of the three vanilla NN estimates provides good error, and an improved error if two or more of the estimates provide vanishing error. The starting point of building such an estimator is a calculation akin to \cref{eq:dr_guess}. Suppose we consider a simplified noiseless setting where $ Y_{j', t', a'} = u_{j'}v_{t'}w_{a'}$ and we are missing the entry $Y_{i, t, a}=u_i v_t w_a$. Then our goal is to find an estimator $\what\theta_{i,t,a}^{TR}$ such that
    \begin{align}
        u_i v_t w_a -\what\theta_{i,t,a}^{TR}  &=(u_i-\what u_i)(v_t-\what v_t)(w_a-\what w_a) \\ &= (u_iv_t+\what u_i\what v_t-u_i\what v_t-\what u_i v_t)(w_a-\what w_a)\\
        &=u_iv_tw_a  + \what u_i \what v_t w_a-u_i \what v_t w_a -\what u_i v_tw_a -u_iv_t\what w_a -\what u_i\what v_t \what w_a+u_i\what v_t \what w_a+\what u_i v_t \what w_a,
    \end{align}
    which motivates the following estimator
    \begin{align}
        \what\theta_{i,t,a}^{TR} = \frac{\sum_{j\in\estnbrnotag^{\trm{unit}}_{i, t, a}, \t' \in\estnbrnotag^{\trm{time}}_{i, t, a}, a' \in\estnbrnotag^{\trm{intervention}}_{i, t, a}}Y_{j, t, a} + Y_{i, t', a}  + Y_{i, t, a'}  + Y_{j, t', a'} -  Y_{j, t', a}  - Y_{j, t, a'}  - Y_{i, t', a'}}{|\estnbrnotag^{\trm{unit}}_{i, t, a}| |\estnbrnotag^{\trm{time}}_{i, t, a}| |\estnbrnotag^{\trm{intervention}}_{i, t, a}|}.
    \end{align}
    The analysis of this paper can be extended to show that this estimator is robust to hetergeoneity in the three sets of factors $\sbraces{u_j}$, $\sbraces{v_{t'}}$ and $\sbraces{w_{a'}}$. In particular, the performance of the estimate $\what\theta_{i,t,a}^{TR}$ would be at least as good as the best of the three vanilla NN estimates whenever at least one of them provides a non-trivial error, and better than any of them as long as at least two of them provide a non-trivial error.

    \subsection{Alternative derivations of \drname}
    \label{sec:alter_derive}
    While we provided an intuitive justification of \drname estimator earlier in \cref{sub:formal_algo}, we now show how it can also be derived using other commonly used ideas of bias-correction and orthogonal loss function minimization.

    \paragraph{\drname as bias adjusted vanilla NN}
    Consider the following term appearing in the sum on the numerator of the RHS in the display \cref{eq:obs_estimate_dr}:
    \begin{align}
    \label{eq:three_terms}
        \obs[\n, \t'] + \obs[j, \t] - \obs[j, \t'].
    \end{align}
    Note that the term $\obs[\n, \t']$ can be treated as the NN estimate for unit $\n$ at time $\t$, estimated using the time neighbor $\t'$. With this view point, the term $\obs[j, \t']-\obs[j, \t]$ approximates the bias of the time neighbor $\t'$, for estimating outcomes at time $\t$ (notably the bias is estimated using unit neighbors for unit $i$). Since the term~\cref{eq:three_terms} is equal to the time-NN estimate minus the bias proxy, we can interpret it (and thereby its average, the \drname estimate) as a bias-adjusted time NN estimate, where the unit neighbors are being used to approximate the bias. By flipping the roles of time and unit neighbors, i.e., by treating the term $\obs[j, \t] $ as the unit-NN estimate, and the term $\obs[j, \t']-\obs[\n, \t']$ as the bias approximation, \drname estimate can be interpreted as bias-adjusted unit NN. This construction can be interpreted as a suitable adaptation of the bias-corrected matching estimators for average treatment effects. In particular, \cite{abadie2011bias} used a covariate-based outcome prediction to approximate the bias and subsequently, use the estimate to debias their NN estimate for average treatment effect in a static setting with a single outcome for each unit~\citep{abadie2011bias}. In contrast, here we directly use suitably chosen outcomes for bias adjustment, and provide entry-wise guarantees in a setting with multiple observations for multiple units (without using any covariates).

    \renewcommand{\u}{\lunit[]}
    \renewcommand{\v}{\ltime[]}
    \paragraph{\drname as minimizer of an orthogonal loss function}
    Since our main goal is estimating $\trueobs$, the latent factors $\lunit$ and $\ltime$ serve as nuisance parameters. We can define a loss function for estimating the parameter of interest as a function of the nuisance parameters: 
    \begin{align}
        L_\star(\trueobs, (\lunit[], \ltime[])) = (\trueobs-\lunit[]\tp \ltime[])^2
    \end{align}
    A zeroth-order approximation of this loss function is given by $L_0(\trueobsvar, (\lunit[], \ltime[])) = (\trueobsvar-\u\tp\v)^2$. So that for the nuisance estimates $\uhat$ and $\vhat$, the estimate for $\trueobs$ satisfies
    \begin{align}
        \arg\min_{\theta} L_0(\theta, (\uhat,\vhat)) = \uhat\tp\vhat.
    \end{align}
     A first-order approximation to the true loss, also called the orthogonal loss~\cite[Eq.~(1)]{dao2021knowledge}, is given by
    \begin{align}
        L_1(\trueobsvar, (\lunit[], \ltime[])) &\defeq  L_0(\trueobsvar, (\lunit[], \ltime[])) + \angles{\grad_{\u, \v} L_0(u, v), \begin{bmatrix}
            \lunit-\u \\ 
            \ltime-\v
        \end{bmatrix} } \\ 
        &= (\trueobsvar-\u\tp\v)^2 + 2(\u\tp\v-\trueobsvar) \angles{\begin{bmatrix}
            \v \\ \u
        \end{bmatrix}, \begin{bmatrix}
            \lunit-\u \\ 
            \ltime-\v
        \end{bmatrix}} \\ 
        &= (\trueobsvar-\u\tp\v)^2 + 2(\u\tp\v-\trueobsvar) (\v\tp\lunit+ \u\tp\ltime-2\u\tp\v)
    \end{align}
    Now, given the nuisance estimates $\uhat$ and $\vhat$, the minimizer of the $L_1$ for $\trueobs$ satisfies
    \begin{align}
        \arg\min_{\theta}L_1(\theta, (\uhat,\vhat)) = \uhat\ltime + \lunit\vhat -\uhat\vhat,
    \end{align}
    which matches the expression in \cref{eq:dr_first}. Thus, this construction can be treated as a particular instance of orthogonal machine learning, adapted to panel data settings.

    \subsection{Consequences for counterfactual inference and causal panel data problems}
    \label{sub:consequence_causal}
    We now provide a brief illustration of how that our model and estimates are relevant also for counterfactual inference in panel data settings and on-adaptive experiments. For example, consider a setting with $\N+1$ users that are assigned multiple binary treatments taking value in $\sbraces{0, 1}$ over $\T+1$ decision times. In particular, suppose $\miss$ denotes the  value of the binary treatment assigned to unit $\n \in[\N+1]$ at time $\t\in[\T+1]$ and $\obs$ denotes the outcome observed for unit $\n$ after the treatment was assigned at time $\t$. Moreover, suppose that for each user at each time there exist two mean potential outcomes, one of which is observed up to some noise as per the following outcome model:
\begin{talign}
\label{eq:model_mc_causal_ext}
    \obs^{(a)} = \trueobsvar_{\n, \t}^{(a)} + \noiseobs
    \qtext{for} a\in\actionset
    \quad \qtext{and}
    \obs = \obs^{(\miss)}.
\end{talign}
Here for each treatment $a\in\sbraces{0, 1}$, the tuple $(\obs^{(a)}, \trueobsvar_{\n, \t}^{(a)})$ denotes the counterfactual outcome (also referred to as the potential outcome~\cite{neyman,rubin1976}) and its mean under treatment $\miss=a$ and $\noiseobs$ denotes exogenous and \iid zero mean noise with variance $\sigma^2$. 
It is easy to see that \cref{eq:model_mc_causal_ext} can be treated as two instances of the problem described in \cref{eq:model_mc}. That is, the estimates and the associated guarantees derived in this work directly apply for the task of estimating both $\trueobsvar_{\n, \t}^{(0)}$ and $\trueobsvar_{\n, \t}^{(1)}$ when the treatments are assigned independently with constant probability $p$. We note that such treatment patterns are commonly used for just in time interventions for experiments on digital platforms, like mobile health and online education. Such studies often form the basis of subsequent adaptive trials and are called exploratory trials~\citep{klasnja2015microrandomized, liao2016sample, nahum2018just, boruvka2018assessing, bell2020notifications}. In prior work, NN-based estimates have been used for counterfactual inference in such settings~\citep{dwivedi2022counterfactual, choi2024counterfactual} and for these problems DR-NN strategy provides an improved set of estimates.
    \paragraph{Relation to synthetic difference in differences (SDID)}
    We note that the SDID estimator~\cite{arkhangelsky2019synthetic} shares similarities with the \drname estimator in this work. More broadly, doubly robust estimation has been studied extensively in the causal inference and missing-data literatures~\citep{robins2000marginal, cao2009improving, bang2005doubly}, and in the off-policy evaluation literature~\citep{jiang2016doubly, kallus2020double, wang2017optimal}. SDID estimates the averaged outcome for a collection of units over a collection of time points, by combining two classical estimators (i) synthetic control and (ii) difference-in-difference and then debiasing them. The estimator has a same form as our \drname, once we map unit-NN as a special case of synthetic control estimator and time-NN as a special case of difference-in-difference estimator. While unit-NN, time-NN, and DR-NN each provide guarantees for a single outcome, the SC and DID estimators provide non-trivial error guarantees for an averaged outcome across time for a given unit and averaged outcome across units for a given time respectively, while SDID provides non-trivial error for averages of multiple units across multiple time points. Moreover, SC, DID and SDID do not handle missing data settings like considered here while NN estimators are well-suited for the settings considered by SC, DID and SDID, as well as the settings considered in this work where entries are stochastically missing (or the treatments are randomized). With these points in mind, we conclude that \drname can be interpreted as a generalization of the SDID estimator to obtain unit-time level treatment effect guarantees for a wide range of treatment patterns.

    \section{Conclusion}
    \label{sec:conclusion}
    In this work, we introduced and analyzed doubly robust variant of nearest neighbors. We provided non-asymptotic and asymptotic guarantees for entry-wise estimation for missing panel data under a latent factor model. We showed that the estimate's double robustness to the underlying heterogeneity in latent factors makes it a superior algorithm compared to the unit-NN and time-NN algorithms.

    The simplicity and interpretability of nearest neighbors makes them attractive approaches for estimation in a range of settings and consequently this work offers several venues of future research. We assumed that the data is missing in an exogenous manner. However, in many settings, like causal panel data, the missingness pattern is often endogenous as the likelihood of observing an outcome depends on the underling mean parameter. In sequential experimental settings, the missingness is determined by a learning policy which takes in historical observations to decide which entries to observe next. Vanilla NN variants or their suitable variations have been shown to be effective in both such settings~\cite{agarwal2021causal, dwivedi2022counterfactual, deb2025counterfactual}. Developing and analyzing suitable analogs of doubly robust NN for such settings is of significant interest; a step in this direction for observational data with causal latent factor models was taken in Abadie et al.~\cite{abadie2024doubly}. Furthermore, we assumed, like prior works on nearest neighbors, that the noise in the observations have identical variances. It is of natural interest to empower NN algorithms to handle heteroskedastic noise~\citep{cai2020uncertainty, chen2019inference}. Recent work has also achieved minimax-optimal rates for NN-based matrix completion~\citep{sadhukhan2025minimax, sadhukhan2025adaptive}, opening avenues for establishing optimality of DR-NN.
    Finally, recent work~\cite{yu2022nonparametric} has established favorable results with nearest neighbors for matrix estimation with side information. In particular, assuming that one set of (unit or time) factors are known, they establish minimax optimality of a suitable variant of NN for a range of settings. A natural future direction is to design NN estimates that are also robust to imperfect knowledge of the side information in such problems. Practical implementations of NN-based matrix completion methods have also been explored in recent software packages~\citep{chin2025unified}.

\appendix 

    \section{{Proof of \textsc{\lowercase{\Cref{thm:anytime_bound}}}: \anytimeboundname}}
    \label{sec:proof_of_anytime_bound}
    
    We start with a basic decomposition of the error into a bias and a variance term in \cref{sub:basic_decomp} and then proceed to bound the two terms respectively in \cref{sub:proof_bias_nonlinear,sub:proof_variance_nonlinearf}.

    \subsection{Basic decomposition}
    \label{sub:basic_decomp}
    Define the shorthands $\mprod \defeq \nmiss[\n, \t'] \nmiss[j, \t] \nmiss[j, \t']$,
    \begin{align}
    \label{eq:theta_hat}
        \what{\theta}&\defeq\frac{\sum_{j \in \estnbr,\t' \in \testnbr} \mprod  (\trueobs[\n, \t']\!+\!\trueobs[j, \t] \!-\! \trueobs[j, \t'])}{\sum_{j \in \estnbr,\t' \in \testnbr} \mprod }
        \qtext{and} \\
    \label{eq:noise_hat}
        \noisehat &\defeq
        \frac{\sum_{j \in \estnbr,\t' \in \testnbr} \mprod  (\noiseobs[\n, \t']\!+\!\noiseobs[j, \t] \!-\! \noiseobs[j, \t'])}{\sum_{j \in \estnbr,\t' \in \testnbr} \mprod }.
    \end{align}
    In the proof, we would account for the probability that a certain event holds, and under that event
    \begin{align}
    \label{eq:mit_nz}
        \sum_{j \in \estnbr,\t' \in \testnbr} \mprod >0,
    \end{align}
    so that we assume that \cref{eq:mit_nz} holds going forward.
    Then by definition~\cref{eq:obs_estimate_dr,eq:model_mc}, we have
    \begin{align}
    \begin{split}
        \estobsdrs = \what{\theta} + \what{\vareps}
        \implies
        (\trueobs-\estobsdrs)^2
        &= (\trueobs-\what{\theta})^2 + \noisehat^2 + 2 (\trueobs-\what{\theta}) \noisehat \\
        &\leq 2(\trueobs-\what{\theta})^2 + 2\noisehat^2.
        \label{eq:det_final}    
    \end{split}
    \end{align}
    We claim that
    \begin{align}
    \label{eq:bias_bnd}
        \P\brackets{(\trueobs\!-\!\what{\theta})^2 \leq
        16L_H^2\brackets{\biggparenth{\frac{\threshold_{1,f}'+ \frac{2\errtermf}{\p\sqrt{\T}} }{c_{f,1}}}^{4/\alpha_1}+\biggparenth{\frac{\threshold_{2,f}'+ \frac{2\errtermf}{\p\sqrt{\N}} }{c_{f,2} }}^{4/\alpha_2}} \big\vert \ulf, \vlf }
        \geq 1-\delta,
    \end{align}
    and
    \begin{align}
    \label{eq:variance_bnd}
        &\P\brackets{\noisehat^2 \leq \frac{c\log(\N\T/\delta)\sigma^2}{\p} \!\biggparenth{ \frac{1}{\nunitnbr[\thresop]}
        \!+\! \frac{1}{\ntimenbr[\threstp]} \!+\! \frac{1}{\p^2\nunitnbr[\thresop]\ntimenbr[\threstp]}
        } \big\vert \ulf, \vlf }  \\
        &\geq 1\!-\!\delta
         \!-\!2\nunitnbr[\thresop] e^{-c'p^2\nunitnbr[\thresop]}
        \!-\!2\ntimenbr[\threstp] e^{-c'p^2\ntimenbr[\threstp]}
        \!-\!e^{-c'p^3 \nunitnbr[\thresop]\ntimenbr[\threstp]},
    \end{align}
    which in turn implies the claimed result in the theorem.

    We now prove each of the bounds separately.

\subsection{Proof of the bound~\cref{eq:bias_bnd} on bias}
\label{sub:proof_bias_nonlinear}
    We have
    \begin{align}
        \trueobs-\what{\theta} &= \frac{\sum_{j \in \estnbr,\t' \in \testnbr} \mprod  (\trueobs- (\trueobs[\n, \t']\!+\!\trueobs[j, \t] \!-\! \trueobs[j, \t']))}{\sum_{j \in \estnbr,\t' \in \testnbr} \mprod } \\
        &\seq{\cref{eq:bilinear}} 
        \frac{\sum_{j \in \estnbr,\t' \in \testnbr} \mprod \parenth{f(\lunit[i],\ltime[t]) - f(\lunit[i],\ltime[t']) - f(\lunit[j],\ltime[t]) + f(\lunit[j],\ltime[t']) }  } {\sum_{j \in \estnbr,\t' \in \testnbr} \mprod }.
        \label{eq:basic_decomp_bias}
    \end{align}
    Let $\nabla_u f$ and $\nabla_v f$ denote the partial derivatives of $f$ with respect to $u$ and $v$ respectively, and let $\nabla^2_{uu} f, \nabla^2_{vv} f, \nabla^2_{uv}f$ denote the second-order partial derivatives with respect to $(u, u)$, $(v, v)$, and $(u, v)$ respectively, such that 
    \begin{align}
        \nabla^2 f = \begin{bmatrix}
            \nabla^2_{uu} f & \nabla^2_{uv} f \\ 
            \nabla^2_{vu} f & \nabla^2_{vv} f
        \end{bmatrix}.
    \end{align}
    Taylor's theorem yields that for some $u'_i, v'_t, u_{i}'', v_t''$, we have
    \begin{align}
        f(u_j, v_t) &= f(u_i, v_t) + (u_j - u_i)\tp \nabla_u f(u_i, v_t) + (u_j - u_i)\tp \nabla^2_{uu} f(u_i', v_t) (u_j - u_i)  \\ 
        f(u_i, v_{t'}) &= f(u_i, v_t) + (v_{t'} - v_t)\tp \nabla_v f(u_i, v_t) + (v_{t'} - v_t)\tp \nabla^2_{vv} f(u_i, v_t') (v_{t'} - v_t)\ \\ 
        f(u_j, v_{t'}) &= f(u_i, v_t) + (u_j - u_i)\tp \nabla_u f(u_i, v_t) + (v_{t'} - v_t)\tp \nabla_v f(u_i, v_t) \\ 
        &
        \qquad+ (u_j - u_i)\tp \nabla^2_{uu} f(u''_i, v_t'') (u_j - u_i) + (v_{t'} - v_t)\tp \nabla^2_{vv} f(u''_i, v_t'') (v_{t'} - v_t) 
        \\ 
        &
        \qquad+2(u_j - u_i)\tp \nabla^2_{uv} f(u''_i, v_t'')(v_{t'} - v_t). 
    \end{align}
    Note that in the above display $u_i'$ (respectively $v_t$) is in fact some convex combination of $u_i$ and $u_j$ (respectively $v_t$ and $v_t'$). Furthermore, the vector $(u_i\tp, v_t\tp)$ is also some convex combination of $(u_i\tp, v_t\tp)\tp$ and $(u_j\tp, v_{t'}\tp)\tp$.
    Some algebra with these expressions yields that
        \begin{align}
            &f(\lunit[i],\ltime[t]) - f(\lunit[i],\ltime[t']) - f(\lunit[j],\ltime[t]) + f(\lunit[j],\ltime[t']) \\ 
            &= (u_j-u_i)\tp \parenth{\nabla_{uu}^2 f(u_i'', v_t'')-\nabla_{uu}^2 f(u_i', v_t)}(u_j-u_i)\\
            &\qquad + (v_t'-v_t)\tp \parenth{\nabla_{vv}^2 f(u_i'', v_t'') - \nabla_{vv}^2 f(u_i, v_t')} (v_t'-v_t) \\
            &\qquad- 2(u_j-u_i)\tp \nabla_{uv}^2 f(u_i'', v_t'') (v_t'-v_t),
        \end{align}
        By \cref{assum:non_linear_f}, we have $\sup_{u,v}\opnorm{\nabla^2 f(u, v)} \leq L_H$, which when put together with the above display implies that
        \begin{align}
            \sabss{\trueobs\!-\!\what{\theta}}
            &=
            \abss{f(u_j, v_t) + f(u_i, v_{t'}) - f(u_j, v_{t'})
            - f(u_i, v_t)} \\
            &\leq 2L_H \sparenth{\stwonorm{u_i-u_j}+\stwonorm{v_t-v_{t'}}}^2.
            \label{eq:bias_sum_squared}
        \end{align}

    The next lemma establishes a suitable concentration results on the distances used to estimate the nearest neighbors. See \cref{sub:proof_of_f_distance_conc} for the proof.

    \newcommand{\concresultname}{Concentration of NN distances}
    \begin{lemma}[\concresultname]
        \label{lemma:f_dist_noise_conc}
        Recall the definitions~\cref{eq:strong_convexity} of $\rho^{\unittag}_{u, u'}$ and $\rho^{\timetag}_{\ltime[], \ltime[]'
    }$. Then, under the setting of \cref{thm:anytime_bound}, for any tuple $(i, t)$ and any fixed $\delta\in (0, 1]$, the event $\event[\trm{dist}]$
    \begin{align}
        \sup_{j\neq\n}\sabss{\rhounit\!-\!\rho^{\unittag}_{\lunit, \lunit[j]} \!-\!2\sigma^2} &\sless{(a)} \frac{\errtermf}{\p\sqrt{\T}}
        \qtext{and}
        \sup_{\t'\neq\t}\sabss{\rhotime-\rho^{\timetag}_{\ltime, \ltime[t']}\!-\!2\sigma^2 } \sless{(b)} \frac{\errtermf}{\p\sqrt{\N}},
        \label{eq:f_event_dist} 
    \end{align}
    holds with probability at least $1-\delta$ conditional on $\ulf,\vlf$,
    where $\errtermf$ is defined in \cref{eq:f_eta_chi}.
    \end{lemma}
    Note that the condition~\cref{eq:strong_convexity} implies that
    \begin{align}
        \ltwonorm{\lunit- \lunit[j]}^2 \leq (\rho^{\unittag}_{\lunit, \lunit[j]}/c_{f,1})^{2/\alpha_1}
        \,\stext{and} \,
        \ltwonorm{\ltime- \ltime[t']}^2 \leq (\rho^{\timetag}_{\ltime, \ltime[t']}/c_{f,2})^{2/\alpha_2},
    \end{align}
     and hence on the event $\event[\trm{dist}]$, we obtain that
    \begin{align}
        \max_{j \in \estnbr}\!\!\ltwonorm{\lunit- \lunit[j]}^2 &\leq \max_{j \in \estnbr} \parenth{\frac{\rho^{\unittag}_{\lunit, \lunit[j]}}{c_{f,1}}}^{2/\alpha_1} \\
         &\sless{\cref{eq:f_event_dist}}  \max_{j \in \estnbr} \biggparenth{\frac{\rhounit -2\sigma^2 + 
           \frac{\errtermf}{\p\sqrt{\T}}} {c_f} }^{2/\alpha}\\ 
        &\sless{\cref{eq:reliable_nbr}} \biggparenth{\frac{\threshold_1-2\sigma^2 + \frac{\errtermf}{\p\sqrt{\T}} }{c_{f,1}}}^{2/\alpha_1} \seq{\cref{eq:eta_cond_f}} 
        \biggparenth{\frac{\threshold_{1,f}'+ \frac{2\errtermf}{\p\sqrt{\T}} }{c_{f,1}}}^{2/\alpha_1}.
    \label{eq:u_bound_f}
    \end{align}
    Similarly, we also find that 
    \begin{align}
    \label{eq:v_bound_f}
        \max_{j \in \testnbr}\!\!\ltwonorm{\ltime- \ltime[\t']}^2 \leq \biggparenth{\frac{\threshold_{2,f}'+ \frac{2\errtermf}{\p\sqrt{\N}} }{c_{f,2}}}^{2/\alpha_2},
    \end{align}
    on the event $\event[\trm{dist}]$. Putting these bounds together with \cref{eq:bias_sum_squared}, we conclude that
    \begin{align}
        (\trueobs\!-\!\what{\theta})^2 \leq
        16L_H^2\brackets{\biggparenth{\frac{\threshold_{1,f}'+ \frac{2\errtermf}{\p\sqrt{\T}} }{c_{f,1}}}^{4/\alpha_1}+\biggparenth{\frac{\threshold_{2,f}'+ \frac{2\errtermf}{\p\sqrt{\N}} }{c_{f,2} }}^{4/\alpha_2}}.
        \label{eq:bias_bnd_nonlinear}
    \end{align}
    with probability $1-\delta$ as claimed in \cref{eq:bias_bnd} and we are done.

    \subsection{Proof of the bound~\cref{eq:variance_bnd} on variance}
    \label{sub:proof_variance_nonlinearf}
    Define 
    \begin{align}
    \begin{split}
        \label{eq:vareps_decomp}
        \noisehat_1 &\defeq  \frac{\sum_{\t' \in \testnbr} \noiseobs[\n, \t'] \nmiss[\n, \t']  \parenth{  \sum_{j \in \estnbr} \nmiss[j, \t] \nmiss[j, \t']} }{ \sum_{\t' \in \testnbr}  \nmiss[\n, \t'] \parenth{  \sum_{j \in \estnbr} \nmiss[j, \t] \nmiss[j, \t']} },
         \\
        \noisehat_2 &\defeq  \frac{\sum_{j \in \estnbr} \noiseobs[j, \t] \nmiss[j, \t] \parenth{ \sum_{\t' \in \testnbr} \nmiss[\n, \t'] \nmiss[j, \t']} }{\sum_{j \in \estnbr} \nmiss[j, \t] \parenth{ \sum_{\t' \in \testnbr} \nmiss[\n, \t'] \nmiss[j, \t']}},  \\
        \qtext{and} \noisehat_3 &\defeq  \frac{\sum_{j \in \estnbr, \t' \in \testnbr} \nmiss[j, \t] \nmiss[\n, \t'] \nmiss[j, \t'] \noiseobs[j, \t']}   {\sum_{j \in \estnbr, \t' \in \testnbr} \nmiss[j, \t] \nmiss[\n, \t'] \nmiss[j, \t']}.
    \end{split}
    \end{align}
    Then by definition~\cref{eq:noise_hat}, we find that
    \begin{align}
    \label{eq:vareps_decomp_2}
        \noisehat = \noisehat_1 + \noisehat_2 - \noisehat_3
        \implies \noisehat^2 \leq 3(\noisehat_1^2 + \noisehat_2^2 + \noisehat_3^2).
    \end{align}
    Let $\mfk A \defeq \sbraces{\nmiss[j, \t'], \nmiss[j, \t], \nmiss[\n, \t'], j \in \estnbr, \t' \in \testnbr}$. Then note that $\mfk A$ is independent of the noise $\mfk Z \defeq \sbraces{\noiseobs[j, \t'], \noiseobs[j, \t], \noiseobs[\n, \t'], j \in \estnbr, \t' \in \testnbr}$ due to the data-split and the fact that the noise is exogenously generated. Consequently, conditioned on $\mfk A$, using Berstein's inequality (and ignoring the second higher order term),\footnote{If $|X_i-\mu_i|<b$ a.s., $\Var(X_i) \leq \sigma^2$, then $\P(|\frac1n|\sum_{i=1}^n( X_i-\mu_i)| \geq t)| \leq 2\exp(-nt^2/(\sigma^2+bt/3))$.} we find that
    \begin{align}
        \noisehat_1^2 &\leq c\sigma^2 \log(6\ntestnbr/\delta) \frac{\sum_{\t' \in \testnbr} \nmiss[\n, \t']  \bigparenth{  \sum_{j \in \estnbr} \nmiss[j, \t] \nmiss[j, \t']}^2 }{ \big(\sum_{\t' \in \testnbr} \nmiss[\n, \t']  \bigparenth{  \sum_{j \in \estnbr} \nmiss[j, \t] \nmiss[j, \t']}\big)^2 }, \\
        \noisehat_2^2 &\leq c\sigma^2 \log(6\nestnbr/\delta) \frac{\sum_{j \in \estnbr} \nmiss[j, \t]  \bigparenth{  \sum_{\t' \in \testnbr} \nmiss[\n, \t'] \nmiss[j, \t']}^2 }{ \big(\sum_{j \in \estnbr} \nmiss[j, \t]  \bigparenth{  \sum_{\t' \in \testnbr} \nmiss[\m, \t'] \nmiss[j, \t']}\big)^2 }, \\
        \noisehat_3^2 &\leq \frac{c\sigma^2 \log(6\nestnbr\ntestnbr/\delta)} {\sum_{j \in \estnbr, \t' \in \testnbr} \nmiss[j, \t] \nmiss[\n, \t'] \nmiss[j, \t']},
    \end{align}
    with probability $1-\delta$. 
    Moreover note that the variables $\nmiss[j, \t'], \nmiss[j, \t], \nmiss[\n, \t']$ in $\mfk A$ are independent of the sets $\estnbr,\testnbr$ due to the data-split, and hence \cref{lem:bc_bound} implies that
    \begin{align}
         \P\brackets{\frac{\sum_{j \in \estnbr} \nmiss[j, \t] \nmiss[j, \t']}{\p^2 \nestnbr }  \in \brackets{\half, \frac32} \vert \estnbr, \ulf, \vlf} &\geq 1- 2\exp\parenth{-\frac{p^2\nestnbr}{8}} \stext{for any $\t'$,} 
         \label{eq:unit_evt}
         \\
         \P\brackets{\frac{\sum_{\t' \in \testnbr} \nmiss[\n, \t']}{\p \ntestnbr }\geq \half \vert \testnbr , \ulf, \vlf} &\geq 1- \exp\parenth{-\frac{p\ntestnbr}{8}},
         \label{eq:time_evt}\\ 
         \qtext{and}
         \P\brackets{\frac{\sum_{j \in \estnbr, \t' \in \testnbr} \nmiss[j, \t] \nmiss[\n, \t'] \nmiss[j, \t']}{\p^3  \nestnbr \ntestnbr }\geq \half \vert \testnbr, \ulf, \vlf} &\geq 1- \exp\parenth{-\frac{p^3\nestnbr\ntestnbr}{8}}.
         \label{eq:unit_time_evt}
    \end{align}
    Using the probability bound~\cref{eq:unit_evt}, and a standard union bound argument, we find that
    \begin{align}
        \frac{\sum_{\t' \in \testnbr} \nmiss[\n, \t']  \bigparenth{  \sum_{j \in \estnbr} \nmiss[j, \t] \nmiss[j, \t']}^2 }{ \bigg(\sum_{\t' \in \testnbr} \nmiss[\n, \t']  \bigparenth{  \sum_{j \in \estnbr} \nmiss[j, \t] \nmiss[j, \t']}\bigg)^2 }
        &\leq \frac{\sum_{\t' \in \testnbr} \nmiss[\n, \t']  \bigparenth{\frac32\p^2\nestnbr}^2 }{ \bigg(\sum_{\t' \in \testnbr} \nmiss[\n, \t'] \half \p^2\nestnbr \bigg)^2 } \\
        &= \frac{9}{\sum_{\t' \in \testnbr} \nmiss[\n, \t']} 
        \leq \frac{18}{\p\ntestnbr},
    \end{align}
    with probability at least
    \begin{align}
        1- \exp\parenth{-\frac{p\ntestnbr}{8}}-2 \nestnbr \exp\parenth{-\frac{p^2\nestnbr}{8}}.   
    \end{align} 
    A similar computation yields that
    \begin{align}
        \frac{\sum_{j \in \estnbr} \nmiss[j, \t]  \bigparenth{  \sum_{\t' \in \testnbr} \nmiss[\n, \t'] \nmiss[j, \t']}^2 }{ \big(\sum_{j \in \estnbr} \nmiss[j, \t]  \bigparenth{  \sum_{\t' \in \testnbr} \nmiss[\m, \t'] \nmiss[j, \t']}\big)^2 } \leq \frac{18}{\p\nestnbr},
    \end{align}
    with probability at least
    \begin{align}
        1-\exp\parenth{-\frac{p\nestnbr}{8}}-2 \ntestnbr \exp\parenth{-\frac{p^2\ntestnbr}{8}}.
    \end{align}
    Putting together the pieces, we find that conditional on $\ulf, \vlf$, we have
    \begin{align}
        \noisehat_1^2 \leq \frac{18c\sigma^2 \log(\frac{6\ntestnbr}{\delta})}{\p\ntestnbr}, 
        \quad
        \noisehat_2^2 \leq \frac{18c\sigma^2 \log(\frac{6\nestnbr}{\delta})}{\p\nestnbr},
        \qtext{and}
        \noisehat_3^2 \leq \frac{2c\sigma^2 \log(6\nestnbr\ntestnbr/\delta)} {\p^3 \nestnbr\ntestnbr},
    \end{align}
    with probability at least
    \begin{align}
        1-\delta-e^{-{p\nestnbr}/{8}}-2\ntestnbr e^{-{p^2\ntestnbr}/{8}}
        -e^{-\frac{p\nestnbr}{8}}-2 \ntestnbr e^{-{p^2\ntestnbr}/{8}}
        -e^{-p^3\nestnbr\ntestnbr/8},
    \end{align}
    which can be further simplified (by lower bounding) as
    \begin{align}
        1\!-\!\delta\!-\!(2\ntestnbr\!+\!1) e^{-{p^2\ntestnbr}/{8}}
        \!-\!(2\ntestnbr\!+\!1) e^{-{p^2\ntestnbr}/{8}}
        \!-\!e^{-p^3\nestnbr\ntestnbr/8},
    \end{align}

    Next, we note that under event $\event[\trm{dist}]$, we have
    \begin{align}
         \unitnbr \subseteq \estnbr
        \qtext{and}
        \timenbr \subseteq \testnbr,
    \end{align}
    and since the data split was done using Bernoulli$(\half)$ flips, we obtain that
    \begin{align}
    \begin{split}
    \label{eq:nn_count}
        \P\brackets{\nestnbr \geq \frac{\nunitnbr[\thresop]}{4} \vert \ulf, \vlf} &\geq 1-\exp\parenth{-\frac{\ntimenbr[\threstp]}{16}}
        \\
        \qtext{and} 
        \P\brackets{\ntestnbr \geq \frac{\ntimenbr[\threstp]}{4} \vert \ulf, \vlf } &\geq 1-\exp\parenth{-\frac{\ntimenbr[\threstp]}{16}}.
    \end{split}
    \end{align}
    Putting the pieces together, we conclude that
    \begin{align}
        \noisehat^2 \leq \frac{c\log(\N\T/\delta)\sigma^2}{\p} \!\biggparenth{ \frac{1}{\nunitnbr[\thresop]}
        \!+\! \frac{1}{\ntimenbr[\threstp]} \!+\! \frac{1}{\p^2\nunitnbr[\thresop]\ntimenbr[\threstp]}
        },
    \end{align}
    with probability at least
    \begin{align}
        1\!-\!\delta
        \!-\!2\nunitnbr[\thresop] e^{-c'p^2\nunitnbr[\thresop]}
        \!-\!2\ntimenbr[\threstp] e^{-c'p^2\ntimenbr[\threstp]}
        \!-\!e^{-c'p^3 \nunitnbr[\thresop]\ntimenbr[\threstp]},
    \end{align}
    for some universal $c, c'>0$, conditional on $\ulf, \vlf$ as claimed.
    \subsection{Proof of \cref{lemma:f_dist_noise_conc}: \concresultname}
    \label{sub:proof_of_f_distance_conc}
    We establish that the bound stated in \cref{eq:f_event_dist}(a) holds with probability at least $1-\delta/2$. Proof of \cref{eq:f_event_dist}(b) can be derived in an analogous manner so that the claim \cref{eq:f_event_dist} holds with probability at least $1-\delta$ as desired.

    Define 
    \begin{align}
        \rhounitstar &\defeq \frac{1}{T} \sum_{t'\neq t} (f(\lunit, \ltime[t']) - f(\lunit[j], \ltime[t']))^2 + 2\sigma^2 \\ 
        \qtext{and}
        \rhotimestar &\defeq \frac{1}{N} \sum_{j \neq i}(f(\lunit[j], \ltime) - f(\lunit[j], \ltime[t']))^2 + 2\sigma^2.
    \end{align}
    Given the data-split, we can express $\rhounit$ as follows:
    \begin{align}
        \rhounit = \frac{\sum_{t'\neq t} B_{t'} \nmiss[\n, \t'] \nmiss[j, \t'] (\obs[n, \t']-\obs[j,\t'])^2 }{ \sum_{t'\neq t}B_{t'} \nmiss[\n, \t'] \nmiss[j, \t']},
    \end{align}
    where $B_{t'}$ denotes the indicator random variable for whether $t'$ is included in the training set or not. Note that $B_{t'}$ are drawn in an \iid manner (and independent of all the randomness in observed data) from a Bernoulli distribution with probability parameter $\half$. 
    Define the shorthands
    \begin{align}
        \wtil{\theta}_{\n, j} &\defeq \brackets{\trueobs[\n,\t']-\trueobs[j,\t']}_{\t'\neq\t} \in \real^{\T},
        \quad
        \wtil{\theta}^{(2)}_{\n, j} \defeq \brackets{(\trueobs[\n,\t']-\trueobs[j,\t'])^2}_{\t'\neq\t} \in \real^{\T},
        \quad
        \wtil{A}_{\n, j} \defeq \brackets{B_t'\nmiss[\n, \t'] \nmiss[j, \t']}_{\t'\neq\t} \in \sbraces{0, 1}^{\T}, \\
        \wtil{\noise}_{\n, j} &\defeq \brackets{\noiseobs[\n,\t']-\noiseobs[j,\t'] }_{\t'\neq\t} \in \real^{\T},
        \quad
        \wtil{\noise}^{(2)}_{\n, j} \defeq \brackets{(\noiseobs[\n,\t']-\noiseobs[j,\t'] )^2}_{\t'\neq\t} \in \real^{\T},
        \stext{and}\T_{\n, j} \defeq \sum_{\t'\neq\t}  B_{t'}\nmiss[\n, \t'] \nmiss[j, \t'] \in [0, \T].
    \end{align}
    Note that $\stwonorm{\wtil{A}_{\n, j}} = \sqrt{\T_{\n, j}}$.
    Then we can write
    \begin{align}
        \begin{split}
            \label{eq:rhos_thetas}
            \rhounit &= \frac{1}{\T_{\n, j}} \parenth{\angles{\wtil{A}_{\n, j}, \wtil{\theta}^{(2)}_{\n, j}} + 
            \angles{\wtil{A}_{\n, j}, \wtil{\noise}^{(2)}_{\n, j}} 
            -2 \angles{\wtil{\noise}_{\n, j}, \wtil{A}_{\n, j}\circ \wtil{\theta}_{\n, j}} 
            } \qtext{and}\\
            \rhotimestar &= \frac1{T} \indicator \tp \wtil{\theta}^{(2)}_{\n, j} + 2\sigma^2,
        \end{split}
    \end{align}
    where $\circ$ denotes the entry wise product between two vectors, i.e., for two vectors $a,b \in \real^n$, the i-th entry of $a\circ b$ is given by $a_i b_i$ for $i \in [n]$. We highlight that the random vectors $\noise_{\n, j}$ and $\wtil{A}_{\n, j}$ are mutually independent. 

    Now fix a unit index $j$.
    Note that $\wtil{\noise}^{(2)}_{\n,j}-2\sigma^2 \mathbf 1$ is a mean zero sub-Gaussian random vector with parameter $4\nconst^2$ and $\wtil{\noise}_{\n, j}$ is a mean zero sub-Gaussian random vector with parameter $2\nconst$. Consequently, under \cref{assum:noise} we have
    \begin{talign}
    \begin{split}
    \label{eq:noise_events}
        \P\brackets{|\angles{\wtil{A}_{\n, j}, \wtil{\noise}^{(2)}_{\n,j}} - 2\sigma^2 \T_{\n, j}| \leq 8\nconst^2 \stwonorm{\wtil{A}_{\n, j}} \sqrt{\log(\frac{6}{\delta})}  \ \big \vert \ \wtil{A}_{\n, j}, \ulf, \vlf } &\geq 1-\frac{\delta}{3} \stext{and} \\ 
        \P\brackets{|\angles{\wtil{\noise}_{\n, j}, \wtil{A}_{\n, j}\circ \wtil{\theta}_{\n, j}}| \leq 4\nconst \stwonorm{\wtil{A}_{\n, j}\circ \wtil{\theta}_{\n, j}}  \sqrt{\log(\frac{6}{\delta})} \ \big \vert \ \wtil{A}_{\n, j}, \ulf, \vlf} &\geq 1-\frac{\delta}{3}.
    \end{split}
    \end{talign}
    Let $m\defeq \frac{p^2}{2} \indicator\tp\wtil{\theta}^{(2)}_{\n, j}$, then we claim that
    \begin{talign}
        \P\brackets{ \abss{\sangles{\wtil{A}_{\n, j}, \wtil{\theta}^{(2)}_{\n, j}}\!-\!m } \!\leq\! \ybnd\sqrt{6m\log(\frac 9\delta)} \stext{and} \T_{\n, j}\!-\!\frac{\p^2\T}{2} \!\geq\! p\sqrt{2T\log(\frac9\delta)} \ \big \vert \ \mc U, \mc V } 
        \!\geq\! 1\!-\!\frac{\delta}{3}.
        \label{eq:f_sig_conc}
    \end{talign}
    We defer the proof of this claim to the end of this section.
    Note that when $p^2\T \geq 32\log(9/\delta)$ as assumed in \cref{eq:eta_cond_f}, then under the event in \cref{eq:f_sig_conc}, we have
    \begin{align}
    \label{eq:t_lower_bnd}
         \T_{\n, j} \geq p^2\T/4.
     \end{align} 
    Next, we note that
    \begin{align}
        &\frac{1}{\T_{\n, j}}\angles{\wtil{A}_{\n, j}, \wtil{\theta}^{(2)}_{\n, j}} - \frac{\indicator\tp\wtil{\theta}^{(2)}_{\n, j}}{\T} = \frac{\angles{\wtil{A}_{\n, j}, \wtil{\theta}^{(2)}_{\n, j}}-\frac{p^2}{2} \indicator\tp\wtil{\theta}^{(2)}_{\n, j}}{\T_{\n, j}} + \frac{\T_{\n,j}-\frac{p^2}{2}\T}{\T_{\n, j}} \frac{ \indicator\tp\wtil{\theta}^{(2)}_{\n, j}}{\T},
        \label{eq:wtilA_theta_reln}\\
        \label{eq:wtilA_theta_circ_norm}
        &\twonorm{\wtil{A}_{\n, j}\circ \wtil{\theta}_{\n, j}} \leq \stwonorm{\wtil{A}_{\n, j}} \sinfnorm{\wtil{\theta}_{\n, j}} \leq 2\ybnd \stwonorm{\wtil{A}_{\n, j}},
        \qtext{and} 
        \frac{ \indicator\tp\wtil{\theta}^{(2)}_{\n, j}}{\T}
        \leq 2\ybnd^2.
    \end{align}
    Putting the events, we find that the intersection of events defined in \cref{eq:noise_events,eq:f_sig_conc} occurs with probability at least $1-\delta$, and on this event denoted by $\event'$,  we have
    \begin{align}
        &|\rhounit-\rhounitstar| \\
        &\leq \bigg| \frac{1}{\T_{\n, j}} \parenth{\angles{\wtil{A}_{\n, j}, \wtil{\theta}^{(2)}_{\n, j}} + 
            \angles{\wtil{A}_{\n, j}, \wtil{\noise}^{(2)}_{\n, j}} 
            -2 \angles{\wtil{\noise}_{\n, j}, \wtil{A}_{\n, j}\circ \wtil{\theta}_{\n, j}}} 
        - \parenth{\frac1{\T_{\n,j}} \indicator \tp \wtil{\theta}^{(2)}_{\n, j} + 2\sigma^2
           } \bigg |
         \\ 
        &\sless{\cref{eq:wtilA_theta_reln}} \abss{\frac{\angles{\wtil{A}_{\n, j}, \wtil{\theta}^{(2)}_{\n, j}}-\frac{p^2}{2} \indicator\tp\wtil{\theta}^{(2)}_{\n, j}}{\T_{\n, j}}}
        \!+\! \abss{\frac{1}{\T_{\n, j}} \angles{\wtil{A}_{\n, j}, \wtil{\noise}^{(2)}_{\n,j}} - 2\sigma^2}
         \!+\! \frac{1}{\T_{\n,j}}2\abss{\angles{\wtil{\noise}_{\n, j}, \wtil{A}_{\n, j}\circ \wtil{\theta}_{\n, j}}} 
        \\
         &\qquad\qquad
        + \frac{|\T_{\n,j}-\frac{p^2}{2}\T|}{\T_{\n, j}} \frac{ \indicator\tp\wtil{\theta}^{(2)}_{\n, j}}{\T} \\ 
        &\sless{(i)}  \frac{\ybnd\sqrt{6m\log(\frac9\delta)}}{\T_{\n, j}} + \frac{8\nconst^2 \stwonorm{\wtil{A}_{\n, j}} \sqrt{\log(\frac{6}{\delta})} }{\T_{\n, j}} + \frac{4\nconst \stwonorm{\wtil{A}_{\n, j}\circ \wtil{\theta}_{\n, j}}  \sqrt{\log(\frac{6}{\delta})}}{\T_{\n, j}}+ \frac{|\T_{\n,j}-\frac{p^2}{2}\T|}{\T_{\n, j}}  2\ybnd^2 \\
        &\sless{(ii)} \frac{\ybnd\sqrt{6p^2\T\ybnd^2/2\log(\frac9\delta)}}{p^2\T/4} + \frac{16\nconst^2 \sqrt{\log(\frac{6}{\delta})}}{p\sqrt{\T}} + \frac{16\nconst\ybnd\sqrt{\log(\frac 6\delta)}}{p\sqrt{\T}} + \frac{p\sqrt{2T\log(12/\delta)}}{p^2\T/2}2\ybnd^2 \\
        &\leq\frac{(13\ybnd^2 + 16\nconst^2 +  16\nconst \ybnd)\sqrt{\log(\frac 9\delta)}}{p\sqrt{\T}},\label{eq:eventprime_bnd} 
    \end{align}
    where in step (i), we use the events~\cref{eq:f_sig_conc,eq:noise_events}, and in step~(ii), the bounds~\cref{eq:t_lower_bnd,eq:wtilA_theta_circ_norm}, and the fact that $\stwonorm{\wtil{A}_{\n, j}}^2 = \T_{\n, j}$.
    Note the following deterministic bounds:
    \begin{align}
    \label{eq:ef_det_error}
        \sup_{j \in [\N+1]}|\rhounitstar-2\sigma^2 - \rho^{\unittag}_{\lunit, \lunit[j]}| \leq
        \frac{2\ybnd^2}{\T+1}
        \qtext{and}
        \sup_{t'\in[\T+1]}|\rhotimestar-2\sigma^2 - \rho^{\timetag}_{\ltime, \ltime[t']}| \leq
        \frac{2\ybnd^2}{\N+1},
    \end{align}
    so that on the event $\event'$, we have
    \begin{align}
        |\rhounit-\rho^{\unittag}_{\lunit, \lunit[j]}-2\sigma^2 | 
        &\quad\!\leq\quad \! |\rhounit-\rhounitstar| + |\rhounitstar- \rho^{\unittag}_{\lunit, \lunit[j]}-2\sigma^2 | \\
        &\sless{\eqref{eq:eventprime_bnd},\eqref{eq:ef_det_error}} \frac{(13\ybnd^2 + 16\nconst^2 +  16\nconst \ybnd)\sqrt{\log(\frac 9\delta)}}{p\sqrt{\T}} + \frac{2\ybnd^2}{\T+1} \\
        &\quad\!\sless{(i)}\quad \! \frac{16(\ybnd+\nconst)^2\sqrt{\log(\frac9\delta)}}{p\sqrt{\T}}.
    \end{align}
    Finally, taking a union bound over all $j\neq \n$, replacing $\delta$ by $\delta/(4\N)$ and recalling the definition~\cref{eq:f_eta_chi} of $\errterm$ yields that the bound \cref{eq:f_event_dist}(a) holds with probability at least $1-\delta/2$ as claimed in the beginning of this proof.

    \paragraph{Proof of claim~\cref{eq:f_sig_conc}}
    We make use of the following high probability bound for weighted Bernoulli random variables (proven in \cref{sec:proof_of_lem:bc_bound}):
    \newcommand{\bcbound}{Concentration of sum of weighted Bernoulli random variables}
   \begin{lemma}[\bcbound]
   \label{lem:bc_bound}
       Let $X_{\l} \sim \mrm{Bernoulli}(p_\l)$ for $\l \in[r]$ be a collection of independent random variables and $\sbraces{a_{\l}}_{\l=1}^{r}$ denote arbitrary set of non-negative scalars and $a_\star \geq \max_{\l\in [r]}a_\l$ be any scalar. Consider the random variable $\Psi \defeq \sum_{\l=1}^r a_\l X_\l$ with $\mu\defeq\E[\Psi] = \sum_{\l=1}^r a_\l p_\l$. Then for fixed $\delta\in (0, 1]$, we have
       \begin{align}
       \label{eq:psi_onesided}
           \P\brackets{ \Psi\!-\!\mu \!\leq \!\sqrt{3a_\star\mu \log(1/\delta)} } \!\geq\! 1\!-\!\delta
           \qtext{and}
           \P\brackets{ \Psi\!-\!\mu \!\geq\! \sqrt{2a_\star\mu \log(1/\delta)} } \!\geq\! 1\!-\!\delta,
       \end{align}
        and consequently that
       \begin{align}
       \label{eq:psi_twosided}
           \P\brackets{|\Psi - \mu| \geq \sqrt{3a_\star\mu \log(2/\delta)} } \geq 1-\delta.
       \end{align}
   \end{lemma}
        
    First, we apply \cref{lem:bc_bound} with weights $a_\l$ set equal to components of the vector $ \wtil{\theta}^{(2)}_{\n, j}$ and random variables $X_\l$ set equal to components of the random vector $\wtil{A}_{\n, j}$. Thus in this case, $p_{\l} = p^2/2$ and hence
    \begin{align}
        \E\brackets{\angles{\wtil{A}_{\n, j}, \wtil{\theta}^{(2)}_{\n, j}} \vert \ulf, \vlf} = \frac{p^2}{2} \indicator\tp \wtil{\theta}^{(2)}_{\n, j} = m.
    \end{align}
    Now applying \cref{eq:psi_twosided} with $a_\star = 2\ybnd^2 \geq \sinfnorm{\wtil{\theta}^{(2)}_{\n, j}}$ to $\Psi = \sangles{\wtil{A}_{\n, j}, \wtil{\theta}^{(2)}_{\n, j}}$ yields that 
    \begin{align}
        \P\brackets{\abss{\sangles{\wtil{A}_{\n, j}, \wtil{\theta}^{(2)}_{\n, j}}-m }  \leq \ybnd\sqrt{6m\log(9/\delta)} \vert \ulf,\vlf }\geq 1-\frac{2\delta}{9}.
    \end{align}
    Next, we apply \cref{lem:bc_bound} with all weights $a_\l$ set equal to $1$ so that $a_\star=1$ is a valid choice; and $X_\l$ set equal to the components of $\wtil A_{i, j}$ so that $p_\l =p^2/2$. Applying \cref{eq:psi_onesided}, we conclude that
    \begin{align}
    \label{eq:t_ij_bound}
        \P\brackets{\T_{\n, j}\!-\!\frac{\p^2\T}{2} \!\geq\! p\sqrt{2T\log(9/\delta)} \ \big \vert \ \mc U, \mc V } 
        \!\geq\! 1\!-\!\frac{\delta}{9}.
    \end{align}
    Putting the pieces together immediately yields the claimed bound~\cref{eq:f_sig_conc}.

    \subsection{Proof of \cref{lem:bc_bound}: \bcbound}
    \label{sec:proof_of_lem:bc_bound}
    Note that claim \cref{eq:psi_twosided} follows immediately by a standard union bound with \cref{eq:psi_onesided}. Next, we prove the bounds in display \cref{eq:psi_onesided} by applying \citep[Thms.~1,2]{raghavan1988probabilistic} to the random variable $\Psi_1 = \Psi/a_{\star}$ with $\mu_1 = \mu/a_\star$. For any $\eps>0$, we find that
   \begin{align}
           \P\bigbrackets{ \Psi_1 \geq (1+\eps) \mu_1 } &\sless{\trm{\citep[Thm.~1]{raghavan1988probabilistic}}} \parenth{\frac{e^{\epsilon}}{(1+\epsilon)^{1+\epsilon}}}^{\mu_1}  \sless{(i)} e^{-\mu_1\eps^2/3} \label{eq:tail_bnd_1}\\
           \qtext{and}
           \P\bigbrackets{\Psi_1 \leq (1-\eps) \mu_1 } &\sless{\trm{\citep[Thm.~2]{raghavan1988probabilistic}}} \parenth{\frac{e^{-\epsilon}}{(1-\epsilon)^{1-\epsilon}}}^{\mu_1}
           \sless{(ii)} e^{-\mu_1\eps^2/2},
           \label{eq:tail_bnd_2}
       \end{align}
    where steps (i) and (ii) follow from the following easy to verify inequalities:
    \begin{align}
    \label{eq:exp_fun_bounds}
     \frac{e^\epsilon}{(1+\epsilon)^{1+\epsilon}} \leq e^{-{\epsilon^2}/{3}}
     \qtext{and}
     \frac{e^{-\epsilon}}{(1-\epsilon)^{1-\epsilon}} \leq e^{-{\epsilon^2}/{2}}
     \qtext{for} \epsilon \in [0, 1].
    \end{align}
    Now inverting the tail bounds~\cref{eq:tail_bnd_1,eq:tail_bnd_2}  by setting the quantity on the RHS equal to $\delta$, yields the claimed bounds in display~\cref{eq:psi_onesided}.

    \section{\textup{Proof of \textsc{\lowercase{\Cref{cor:bilinear}}}: \textsc{\bilinearresultname}}}    
    \label{sec:proof_of_cor:bilinear}
    Arguing as in the proof of \cref{thm:anytime_bound}, and in particular using the claims~\cref{eq:det_final,eq:variance_bnd} along with the following analog of \cref{eq:bias_bnd} yields the result of \cref{cor:bilinear}:
    \begin{align}
    \label{eq:bias_bnd_bilinear}
        \P\brackets{(\trueobs\!-\!\what{\theta})^2 \leq
        \frac{1}{\lamtime\lamunit}\parenth{\thresop + \frac{2\errterm}{\p\sqrt{\T}}} \parenth{\threstp + \frac{2\errterm}{\p\sqrt{\N}}}}
        \geq 1-\delta.
    \end{align}
    It remains to establish the claim~\cref{eq:bias_bnd_bilinear}.

    Note the following analog of \cref{eq:basic_decomp_bias} for the bilinear factor model:
    \begin{align}
        \trueobs-\what{\theta} &= \frac{\sum_{j \in \estnbr,\t' \in \testnbr} \mprod  (\trueobs- (\trueobs[\n, \t']\!+\!\trueobs[j, \t] \!-\! \trueobs[j, \t']))}{\sum_{j \in \estnbr,\t' \in \testnbr} \mprod } \\
        &\seq{\cref{eq:bilinear}} 
        \frac{\sum_{j \in \estnbr,\t' \in \testnbr} \mprod \angles{\lunit-\lunit[j], \ltime-\ltime[\t']} } {\sum_{j \in \estnbr,\t' \in \testnbr} \mprod },
    \end{align}
    which together with Cauchy-Schwarz's inequality implies that
    \begin{align}
    \label{eq:prod_bias}
        \sabss{\trueobs\!-\!\what{\theta}} \!\leq\! \max_{j \in \estnbr} \max_{\t' \in \testnbr} |\angles{\lunit\!-\!\lunit[j], \ltime\!-\!\ltime[\t']}|
        \leq \max_{j \in \estnbr} \twonorm{\lunit\!-\!\lunit[j]} \max_{\t' \in \testnbr} \twonorm{\ltime\!-\!\ltime[\t']}.
    \end{align}
    Now applying \cref{lem:bilinear_nonlinear} to the bounds~\cref{eq:u_bound_f,eq:v_bound_f} and using \cref{eq:prod_bias} instead of \cref{eq:bias_sum_squared} immediately yields \cref{eq:bias_bnd_bilinear} instead of \cref{eq:bias_bnd_nonlinear}. The proof is now complete.

\section{Proof of \textsc{\lowercase{\Cref{thm:asymp}}}: \asympresultname}
\label{sec:proof_of_asymp}
    We repeatedly use the following fact: For a sequence of random variables $\sbraces{X_T}$ and deterministic scalars $\braces{b_T}$, we have
    \begin{align}
    \label{eq:op_cond}
         X_T = \order_P(b_T) \qtext{and} b_T = o(1) \implies X_T = o_p(1) \stext{or equivalently} X_T \stackrel{p}{\longrightarrow}0.
     \end{align} 

        \subsection{Proof of part~\cref{item:cons}: Asymptotic consistency} We need to show that $|\estobsdr - \trueobs| = o_p(1)$. Using \cref{eq:op_cond}, the claim follows directly by noting that
        \begin{align}
            |\estobsdr - \trueobs| &= \order_P\bigg[ (\threshold_{1}-2\sigma^2+ \frac{\log(\N+\T)}{\p\sqrt{\T}})^{2/\alpha_1} + (\threshold_{2}-2\sigma^2+ \frac{\log(\N+\T)}{\p\sqrt{\N}})^{2/\alpha_2}  \\ 
            &\qquad\qquad + \frac{\sigma^2}{\p}\biggparenth{ \frac{\log(\stime)}{\sunit}
    + \frac{\log(\sunit)}{\stime} + \frac{\log(\sunit\stime)}{\p^2\sunit\stime\!} }
            \bigg ]
        \end{align}
        due to \cref{thm:anytime_bound} whenever \cref{eq:cons_var_p} holds and that the quantity inside the parentheses in the RHS of the display above is $o(1)$ under the stated assumptions.

        \subsection{Proof of part~\cref{item:clt}: Asymptotic normality}
        We note that the condition~\cref{eq:clt_bias} ensures that the asymptotic bias (scaled by $\sqrt{\effnn}$) is negligible, while the condition~\cref{eq:clt_p_cond} allows us to apply suitable central limit theorem arguments for the variance term.

        Recall the definitions~\cref{eq:theta_hat,eq:noise_hat} of $\what\theta$ and $\noisehat$. Then 
        \cref{eq:det_final} implies that
        \begin{align}
            \sqrt{\effnn} (\estobsdr[i,t, \mbi{\threshold},K] - \trueobs) = \sqrt{\effnn} (\what{\theta} - \trueobs) + \sqrt{\effnn} \noisehat.
        \end{align}
        Next, we have
        \begin{align}
            \sqrt{\effnn} (\what{\theta} - \trueobs )
            &\seq{\cref{eq:bias_bnd}} \order_P\bigg[ \sqrt{\effnn} \parenth{ (\threshold_{1}-2\sigma^2+ \frac{\log(\N+\T)}{\p\sqrt{\T}})^{2/\alpha_1} + (\threshold_{2}-2\sigma^2+ \frac{\log(\N+\T)}{\p\sqrt{\N}})^{2/\alpha_2}} \bigg]\\
            &\seq{\cref{eq:clt_bias}} \order_P(o(1)),
        \end{align}
        which when combined with the next lemma yields the desired claim:
        \newcommand{\cltnoise}{Central limit theorem for suitably scaled noise}
        \begin{lemma}[\cltnoise]
        \label{lem:clt_noise}
            Under the assumptions of \cref{thm:asymp}\cref{item:clt}, we have $ \sqrt{N_{i,t}} \noisehat \Longrightarrow \mc N(0, \sigma^2)$.
        \end{lemma}
        \subsection{Proof of \cref{lem:clt_noise}: \cltnoise}
        Recall the decomposition~\cref{eq:vareps_decomp,eq:vareps_decomp_2} so that $\noisehat=\noisehat_1 + \noisehat_2 + \noisehat_3$.
        Next, we define
        \begin{align}
                \noisetil_1 &\defeq \frac{\sum_{\t' \in \testnbr} \noiseobs[\n, \t'] \nmiss[\n, \t'] }{ \sum_{\t' \in \testnbr}  \nmiss[\n, \t']} 
                \qtext{and}
                \noisetil_2 \defeq \frac{\sum_{j \in \estnbr} \noiseobs[j, \t] \nmiss[j, \t] }{ \sum_{j\in \estnbr}  \nmiss[j, \t]}.
            \end{align}
            Then we make the following claims as long as \cref{eq:clt_p_cond} holds.
            \begin{enumerate}[label=(\Roman*),leftmargin=*]
                \item\label{item:noise_reln_hat_til} \tbf{Relating $(\noisehat_1, \noisehat_2)$ with $(\noisetil_1, \noisetil_2)$}: Conditioned on $\estnbr$ and $\testnbr$, we have
                \begin{align}
                \label{eq:noise_reln_hat_til}
                    \noisehat_1 \seq{(a)} \noisetil_1 + \order_P(\delta),
                    \quad
                    \noisehat_2 \seq{(b)} \noisetil_2 + \order_P(\delta),
                    \qtext{where}
                    \delta  = \frac{\sigma}{\sqrt{p^3 \nestnbr \ntestnbr}}.
                \end{align}
                \item\label{item:til_clt} \tbf{CLT for $(\noisetil_1, \noisetil_2, \noisehat_3)$}: Conditioned on $\estnbr$ and $\testnbr$, we have
                \begin{align}
                \begin{split}
                    \label{eq:til_clt}
                    &\sqrt{\Sigma_{\t' \in \testnbr}  \nmiss[\n, \t']} \noisetil_1 \ \Longrightarrow\  \mc N (0, \sigma^2),
                    \quad
                    \sqrt{\Sigma_{j \in \estnbr}  \nmiss[j, \t]} \noisetil_2 \ \Longrightarrow\  \mc N (0, \sigma^2)
                    ,\qtext{and} \\ 
                    &\sqrt{{\Sigma_{\t' \in \testnbr}\Sigma_{j \in \estnbr}  \nmiss[\n, \t'] \nmiss[j, \t']   \nmiss[j, \t]}} \noisehat_3 \ \Longrightarrow\  \mc N (0, \sigma^2).
                \end{split}
                \end{align}
                \item\label{item:n_it_bound} \tbf{Asymptotic behavior of $\effnn$}: We have
                \begin{align}
                \label{eq:n_it_bound}
                    \frac{\sqrt{\effnn}} {\sqrt{p^3 \nestnbr \ntestnbr}} \inprob 0.
                \end{align}
            \end{enumerate}
            Assuming these claims as given, we finish the proof with the help of following standard result: If $\sqrt{\alpha} \xi_1 \Longrightarrow \mc N(0, \sigma^2)$ and $\sqrt{\beta} \xi_2 \Longrightarrow \mc N(0, \sigma^2)$ and the two sequences are independent of each other, then 
            \begin{align}
            \label{eq:sum_g_clt}
                \sqrt{\frac1\alpha + \frac1\beta} (\xi_1+\xi_2) \implies \mc N(0, \sigma^2).
            \end{align}
            Note that $\noisetil_1$, $\noisetil_2$, and $\noisehat_3$ are mutually independent. Then 
            putting the pieces together yields that
            \begin{align}
                \sqrt{\effnn} \noisehat
                &\seq{\cref{eq:vareps_decomp_2}} \sqrt{\effnn} (\noisehat_1+\noisehat_2+\noisehat_3)\\ 
                &\seq{\cref{eq:noise_reln_hat_til}} \sqrt{\effnn}(\noisetil_1  + \noisetil_2  +  \noisehat_3) + \order_P\biggparenth{\frac{\sigma \sqrt{\effnn}}{\sqrt{p^3 \nestnbr \ntestnbr}}} \\
                &\seq{\cref{eq:n_it_bound}} \sqrt{\effnn}(\noisetil_1  + \noisetil_2  +  \noisehat_3) + o_p(1) 
                \stackrel{\cref{eq:sum_g_clt,eq:til_clt}}{\implies} \mc N(0, \sigma^2),
            \end{align}
            as claimed. It now remains to establish the three parts.

            \paragraph{Proof of part~\cref{item:noise_reln_hat_til}, i.e., display~\cref{eq:noise_reln_hat_til}} We prove relation (a) from display~\cref{eq:noise_reln_hat_til} between $\noisehat_1$ and $\noisetil_1$; the claim for relation (b) between $\noisehat_2$ and $\noisetil_2$ follows from a similar argument.
            
            Condition on the set $\estnbr$ and $\testnbr$ and define $w_{t'} \defeq \sum_{j \in \estnbr} \nmiss[j, \t] \nmiss[j, \t']$ and $w_\star = \p^2\nestnbr$. Let $\zeta_{t'} \defeq w_{t'} - w_\star$, then \cref{lem:bc_bound,eq:t_ij_bound} directly imply that $\zeta_{t'} = \order_P(\sqrt{w_\star})$ uniformly for all $t' \in \testnbr$, which in turn implies also that
            \begin{align}
            \label{eq:w_ratio}
                \frac{\sum_{\t' \in \testnbr}  \nmiss[\n, \t'] \parenth{  \sum_{j \in \estnbr} \nmiss[j, \t] \nmiss[j, \t']}}{\sum_{\t' \in \testnbr}  \nmiss[\n, \t'] w_\star} \inprob 1.
            \end{align}

            Note that given $\testnbr$, the three sets of random variables $\sbraces{\zeta_{t'}}_{\t'\in\testnbr},\sbraces{\noiseobs[\n, \t']}_{\t'\in\testnbr},$ and $\sbraces{\nmiss[\n, \t']}_{\t'\in\testnbr}$ are mutually independent of each other.
            Moreover, we also have that for any fixed vectors $v \in \real_{+}^{\ntestnbr}$ with sum of entries bounded below by a non-zero constant and $v' \in\real^{\ntestnbr}$ we have
            \begin{align}
                \frac{\sum_{t'\in\testnbr}  \nmiss[j, \t'] v_{t'}}{p \sum_{t'\in\testnbr}v_{t'}} \inprob 1 
            \qtext{and}
                \sum_{\t' \in \testnbr} \noiseobs[\n, \t']  v'_{t'} = \order_P(\sigma\stwonorm{v'}).
            \end{align}
            By sequentially removing conditioning on $\sbraces{\zeta_{t'}}$, $\sbraces{\nmiss[\n, \t']}$ and $\sbraces{\nmiss[\n, \t']}_{\t'\in\testnbr}$, we find that
            \begin{align}
                \sum_{\t' \in \testnbr} \noiseobs[\n, \t']  \nmiss[\n, \t'] (w_{t'} - w_\star)
                = \sum_{\t' \in \testnbr} \noiseobs[\n, \t']  \nmiss[\n, \t'] \zeta_{t'} 
                = \order_P(\sigma \sqrt{p \ntestnbr w_\star}),
            \end{align}
            so that
            \begin{align}
                \frac{\sum_{\t' \in \testnbr} \noiseobs[\n, \t']  \nmiss[\n, \t'] (w_{t'} - w_\star)}{\sum_{\t' \in \testnbr} \nmiss[\n, \t'] w_\star} = \order_P\biggparenth{\frac{\sigma \sqrt{p \ntestnbr w_\star}}{p\ntestnbr w_\star}} = \order_P\biggparenth{\frac{\sigma}{\sqrt{p \ntestnbr w_\star}}}.
            \end{align}
            Putting the pieces together, we finally have
            \begin{align}
                 \frac{\sum_{\t' \in \testnbr} \noiseobs[\n, \t'] \nmiss[\n, \t']  \parenth{  \sum_{j \in \estnbr} \nmiss[j, \t] \nmiss[j, \t']} }{ \sum_{\t' \in \testnbr}  \nmiss[\n, \t'] w_\star }
                 -\frac{\sum_{\t' \in \testnbr} \noiseobs[\n, \t'] \nmiss[\n, \t']   }{ \sum_{\t' \in \testnbr}  \nmiss[\n, \t'] }
                 = \order_P\biggparenth{\frac{\sigma}{\sqrt{p \ntestnbr w_\star}}}
            \end{align}
            which when put together with \cref{eq:w_ratio} implies that
            \begin{align}
                \noisehat_1 - \noisetil_1 = \order_P\biggparenth{\frac{\sigma}{\sqrt{p \ntestnbr w_\star}}} = \order_P\biggparenth{\frac{\sigma}{\sqrt{p^3 \nestnbr \ntestnbr}}},
            \end{align}
            as claimed. 

            \paragraph{Proof for part~\cref{item:til_clt}, i.e., display~\cref{eq:til_clt}} Note that $\sbraces{\nmiss[\n, \t']}$ and $\sbraces{\noiseobs[\n, \t']}$ are mutually independent sequences, and hence $\E[\nmiss[\n, \t']\noiseobs[\n, \t']] = 0$ and $\Var(\nmiss[\n, \t']\noiseobs[\n, \t'])=p\sigma^2$. Moreover, we \cref{lem:bc_bound} also implies that $\frac{\sum_{\t' \in \testnbr}  \nmiss[\n, \t']}{p\ntestnbr} \stackrel{p}{\to} 1$ and hence Applying standard central limit theorem~\citep[Thm.~3.4.1]{durrett2019probability} and Slutsky's theorem, we thus have
            \begin{align}
                \sqrt{\Sigma_{\t' \in \testnbr}  \nmiss[\n, \t']} \noisetil_1 = \frac{1}{\sqrt p} \frac{\sum_{\t' \in \testnbr} \noiseobs[\n, \t'] \nmiss[\n, \t']}{\sqrt{\ntestnbr}} \cdot \sqrt{\frac{\ntestnbr}{\Sigma_{\t' \in \testnbr}  \nmiss[\n, \t']}} \implies \mc N (0, \frac{p\sigma^2}{p}),
            \end{align}
            if $\ntestnbr \to \infty$ (which it does due to \cref{eq:in_prob_cgc} and the assumption that $p\sqrt{\stime\to \infty}$.)
            The proof for the other random variables follows similarly, and we are done.

            \paragraph{Proof of part~\cref{item:n_it_bound}, i.e., display~\cref{eq:n_it_bound}}
            Note the following immediate consequences of two-sided versions of the concentration bounds \cref{eq:time_evt,eq:nn_count} (which follows directly from \cref{lem:bc_bound}):
            \begin{align}
            \label{eq:in_prob_cgc}
              \frac{\Sigma_{\t' \in \testnbr}  \nmiss[\n, \t']}{p\ntestnbr}\!\! \inprob\!\! 1\qtext{and}
              \frac{\ntestnbr}{\stime[\threshold_2']}   \!\!\inprob\!\!1
              \implies 
              \frac{\Sigma_{\t' \in \testnbr}  \nmiss[\n, \t']}{p\stime[\threshold_2']} \!\!\inprob\!\! 1.
            \end{align}
            and analogously
            \begin{align}
                \frac{\Sigma_{j \in \estnbr}  \nmiss[j, \t]}{p\nestnbr} \!\!\inprob\!\! 1
                \qtext{and} \frac{\nestnbr}{\sunit[\threshold_1']}   \!\!\inprob\!\!1 
              \implies 
              \frac{\Sigma_{j \in \estnbr} \nmiss[j, \t]}{p\sunit[\threshold_1']} \!\!\inprob\!\! 1.
            \end{align}
            Using the fact that $(\sum_{i=1}^3 \frac1{a_i})\inv \leq (\sum_{i=1}^2 \frac1{a_i})\inv = \frac{a_1a_2}{a_1+a_2}$ for $a_i >0$ for $i\in[3]$, we then have
            \begin{align}
                \effnn &\leq \frac{(\Sigma_{\t' \in \testnbr}  \nmiss[\n, \t'])(\Sigma_{j \in \estnbr}  \nmiss[j, \t])}{\Sigma_{\t' \in \testnbr}  \nmiss[\n, \t']+\Sigma_{j \in \estnbr}  \nmiss[j, \t]} 
                =\order_P\parenth{\frac{p\nestnbr\ntestnbr}{\nestnbr+\ntestnbr}} \\
                 \implies \frac{\sqrt{\effnn}} {\sqrt{p^3 \nestnbr \ntestnbr}} 
                &\leq \frac{1}{p\sqrt{(\nestnbr+\ntestnbr)}} = \order_P\parenth{\frac{1}{p\sqrt{(\sunit[\threshold_1']+(\stime[\threshold_2'])}}}  \seq{\cref{eq:clt_p_cond}} o_p(1).
            \end{align}

\section{Proof of \lowercase{\Cref{cor:anytime_bound}}: \bilinexampleresultname}
\label{proof_of_cor:anytime_bound}
    We prove each part separately.

    \paragraph{Proof of part~\cref{item:finite_anytime}}
        We have $\threstp +  2\frac{\errterm}{p\sqrt{\T}} = \thresop +  2\frac{\errterm}{p\sqrt{\N}}= 4\errtwo $ and hence as noted in \cref{example:finite}, we have $\nunitnbr \geq  \frac{\N}{2\munit} = \frac{\N}{2\M}$ and $\ntimenbr \geq  \frac{\T}{2\mtime} = \frac{\N}{2\M}$ since $\thresop, \threstp = \Omega(\N^{-\half})$. Substituting these quantities in the bound from \cref{thm:anytime_bound}, we find that
        \begin{align}
            (\estobsdr-\trueobs)^2 \leq c \frac{\errterm^2}{\p^2\N} + \frac{c\log(\N^2/\delta)\sigma^2}{\p}\!\biggparenth{\! \frac{2\M}{\N} + \frac{\M^2}{\p^2\N^2} }
            = \otil\parenth{\frac{\M}{\N^{1-\beta}} + \frac{\M^2}{\N^{2-3\beta}}},
        \end{align}
        with the claimed probability. Our result follows.

    \paragraph{Proof of part~\cref{item:highd_anytime}}
        We first derive a general bound for $\p=\Theta(\N^{-\beta})$ the choice $\threshold_1 = \threshold_2 = 2\sigma^2 +\errtwo + r$ for $r= \Omega(\errtwo)$ so that $\thresop = \threstp = \Theta(r)$. As noted in \cref{example:continuous}, we then have
        \begin{align}
            \nunitnbr = \Theta(\N r^{\frac{d}{2}}) \qtext{and} \ntimenbr = \Theta(\N r^{\frac{d}{2}}).
        \end{align}
        Substituting these quantities in the bound from \cref{thm:anytime_bound}, we find that
        \begin{align}
            (\estobsdr-\trueobs)^2 &\leq c r^2 +  c \frac{r\errterm}{\p\sqrt{\N}} + c \frac{\errterm^2}{\p^2\N} + \frac{c}{\p \N r^{\dhalf}} + \frac{c}{\p^3 \N^2 r^{d}} \\ 
            &=\otil\parenth{r^2 + \frac{r}{\N^{\frac{1-2\beta}{2}}} + \frac{1}{\N^{1-\beta} r^{\dhalf}} + \frac{1}{\N^{2-3\beta} r^{d}} }.
        \end{align}
        with probability at least $1-2\delta-2\N e^{-c\N^{1-2\beta} r^{d/2}} - e^{-c \N^{2-3\beta} r^d}$.
        Let $g$ denote the function inside the parentheses in the last display.
        We have
        \begin{align}
            r = \N^{-\frac{1-2\beta}{2}} 
            \implies g(r) &= 
            \order\sparenth{\max\sbraces{\N^{-(1-2\beta)},\ \N^{(1-2\beta)d-4},\ \N^{(\half-\beta)d-(1-3\beta)}} }  \\
            r = \N^{-\frac{2(1-\beta)}{d+4}} \implies 
            g(r) &= \order\sparenth{\max\sbraces{\N^{-\frac{4(1-\beta)}{d+4} },\ \N^{-\frac{4+\dhalf-(6+d)\beta}{d+4}},\ \N^{-\frac{8-\beta(d+12)}{d+4} }} } \\
            r = \N^{-\frac{2-3\beta}{d+2} } \implies g(r) &= \order\sparenth{\max\sbraces{\N^{-\frac{2(2-3\beta)}{d+2}},\  \N^{-\frac{\frac{d+5}{2} (1-2\beta) + \half }{d+2} },\  
            \N^{-\frac{2(1-\beta)+\frac{d\beta}{2}}{d+2} } } }.
        \end{align}
        For any given choice of $\beta\in[0, \half]$, we can choose an $r$ that obtains the minimum value for $g$ to obtain a suitable bound. For $\beta=0$, this choice is $r=\N^{-\frac{2(1-\beta)}{d+4}}$ with $\beta=0$, which in fact yields the claimed result in the corollary.

        \section{Proof of \lowercase{\Cref{cor:dr_gen_improv}}: \improvcorresultname}
    \label{sec:proof_of_cor:dr_gen_improv}
        Under the stated conditions, for $\bthreshold=(\threshold_1,\threshold_2)$ where $\threshold_1 = 2\sigma^2 + \frac{2\errterm}{p\sqrt{T}} + r_1 $ and $\threshold_2= 2\sigma^2 + \frac{2\errterm}{p\sqrt{\N}}+r_2$, we have
        \begin{align}
            \biasdrs = c \bigparenth{\frac{2\errterm}{p\sqrt{\T}} + r_1}^{4/\alpha_1} + c \bigparenth{\frac{2\errterm}{p\sqrt{\N}} + r_2}^{4/\alpha_2}
            \qtext{and}
            \vardrs = c \bigparenth{ \frac{\log(\T r_2^{\beta_2}/\delta)}{\N r_1^{\beta_1}} +  \frac{\log(\N r_{1}^{\beta_1}/\delta)}{\T r_2^{\beta_2}}},
        \end{align}
        We find that optimizing this bound with respect to $r_1$ and $r_2$ yields that
        \begin{align}
            r_1^{\trm{opt}} = \order(\N^{-\frac{1}{4/\alpha_1+\beta_1}})
            \qtext{and}
            r_2^{\trm{opt}} = \order(\T^{-\frac{1}{4/\alpha_2+\beta_2}})
        \end{align}
        which in turn yields that for the corresponding choice of $\bthreshold$, we have
        \begin{align}
            \biasdrs + \vardrs = \otil(\N^{-(\frac{4}{4+\alpha_1\beta_1} \wedge \frac{2}{\alpha_2}) } + \T^{-(\frac{4}{4+\alpha_2\beta_2} \wedge \frac{2}{\alpha_1}) })
        \end{align}
        yielding the claimed bound.

\section{Examples for \textsc{\lowercase{\Cref{assum:non_linear_f}}}}
\label{sec:lip_assum_3}

In the next result, we present two examples that satisfy \cref{assum:non_linear_f}, where we also show that \cref{assum:bilinear_f} is a special case of \cref{assum:non_linear_f}.

\newcommand{\bilinearisnonlinear}{Examples for \cref{assum:non_linear_f}}
\begin{lemma}[\bilinearisnonlinear]
\label{lem:bilinear_nonlinear}
The following statements hold true.
\begin{enumerate}[label=(\alph*)]
	\item\label{item:lin_ex} \tbf{Bilinear $\lfun$:} Consider the setting where \cref{assum:bilinear_f} holds. Then \cref{assum:non_linear_f} holds with $\ybnd = \uconst\vconst$, $L_H = 1$, $c_{f,1} = \lamunit$, $c_{f,2} = \lamtime$, $\alpha_1=\alpha_2 = 2$.
	\item\label{item:non_lin_ex} \tbf{General $\lfun$:} Suppose $f:\real^d \times \real^d \to \real$ is twice-differentiable, $L_1$-smooth with $f(0, 0) = 0$ and $\twonorm{\nabla f(0, 0)} = L_2$. Moreover, there exist constants $\uconst, \vconst, m, alpha$ such that $\twonorm{u} \leq \uconst, \twonorm{v} \leq \vconst$ and $|f(u, v)-f(u', v)| \geq m\twonorm{u-u'}^{\alpha}$ and $|f(u, v)-f(u, v')| \geq m\twonorm{v-v'}^{\alpha}$ for all $u, u'\in \ulf, v, v' \in \vlf$. Then  \cref{assum:non_linear_f} holds with
	 $\ybnd = L_1 (\uconst + \vconst) + \frac {L_2}2(\uconst^2+\vconst^2).$, $L_H = L_1$, $c_{f,1} = c_{f,2} = m^2$, and $\alpha_1=\alpha_2 = 2\alpha$.
\end{enumerate}
\end{lemma}

\begin{proof}
We prove the two parts separately.
\paragraph{Proof of part~\cref{item:lin_ex}}
    First, note that by Cauchy-Schwarz's inequality, we have $f(u_i, v_t)\leq \twonorm{u_i} \twonorm{v_t} \leq \uconst\vconst$, so that $\ybnd = \uconst\vconst$ is a valid choice.

    Next, we can check that $\grad^2 f = \begin{bmatrix}
    	\mbf 0_{d\times d} & \mbf I_{d\times d} \\ 
    	\mbf I_{d\times d} & \mbf 0_{d\times d}
    \end{bmatrix}$ so that $\opnorm{\grad^2 f} = 1 = L_H$. 

    Finally, we note that under \cref{assum:bilinear_f}, the conditions~\cref{eq:nbr_unit} simplify to
    \begin{align}
    \label{eq:nbr_unit_bilinear}
       \rho^{\unittag}_{\lunit[],\lunit[]'} &=  (\lunit[]- \lunit[]')\tp \!\Sigv (\lunit[]-\lunit[]') \geq \lamunit \stwonorm{\lunit[]-\lunit[]'}^2
       \qtext{and} \\
       \rho^{\timetag}_{\ltime[],\ltime[]'} &=  (\ltime[]\!-\!\ltime[]')\tp \!\Sigu (\ltime[]-\ltime[]') \geq \lamtime \stwonorm{\ltime[]-\ltime[]'}^2,
    \end{align}
    which immediately implies \cref{eq:strong_convexity} with the claimed constants $c_{f,1} = \lamunit$, $c_{f,2} = \lamtime$, and $\alpha_1=\alpha_2 = 2$.

\paragraph{Proof of part~\cref{item:non_lin_ex}}
	Note that $L$-smoothness immediately implies that $\opnorm{\grad^2 f} \leq L$ and 
	\begin{align}
		f(u, v) - f(0, 0) \leq 
		\grad f(0, 0) \tp \begin{bmatrix} u \\ v \end{bmatrix} 
		+ \frac{L_2}{2} (\twonorm{u}^2 + \twonorm{v}^2)
		\leq L_1 (\uconst + \vconst) + \frac {L_2}2(\uconst^2+\vconst^2).
	\end{align}
	The remainder of the claim follows by noting that
	\begin{align}
		|f(u, v) - f(u', v)| &\geq m\twonorm{u-u'}^{\alpha} \\ 
		\implies 
		\frac1{\T+1}\sum_{t=1}^{\T+1} (f(u, \ltime)-f(u',\ltime))^2
		&\geq m^2 \twonorm{u-u'}^{2\alpha}.
	\end{align}
  \end{proof}

\begin{acks}[Acknowledgments]
The authors thank Avi Feller and James M. Robins for their helpful comments.
\end{acks}

\bibliographystyle{imsart-number} %
\bibliography{refs}

\end{document}